\def\springer{true} 

\ifdefined\springer
	
	\documentclass[sn-mathphys-num]{style/springer/sn-jnl}
\else
	\documentclass[letterpaper, 10 pt, journal, twoside]{style/IEEEtran}
	
\fi

\usepackage[numbers]{natbib}
\usepackage{booktabs}
\usepackage{moreverb}
\usepackage{url}
\usepackage{amsmath,amssymb}
\usepackage{multicol,lipsum}
\usepackage{bm}
\usepackage{color, colortbl}

\ifdefined\springer
\else
	\usepackage{graphicx}
	\usepackage[colorlinks,bookmarksopen,bookmarksnumbered,citecolor=red,urlcolor=red]{hyperref}
\fi

\usepackage{plain}
\usepackage{setspace}

\usepackage{amsmath,amssymb}
\usepackage{float}
\usepackage{url}
\usepackage{algorithm}
\usepackage{algpseudocode}
\usepackage{multirow}
\usepackage{gensymb}
\usepackage[font=small]{caption}
\usepackage[intoc]{nomencl} 
\usepackage[acronym, nonumberlist]{glossaries}
\usepackage{setspace}
\usepackage{amsthm}
\usepackage{enumitem}

\hypersetup
{
	pdftitle = {Whole-body control for quadrupedal locomotion on challenging terrain},
	pdfauthor = {Shamel Fahmi},
	pdfsubject = {RA-L manuscript},
	pdfkeywords = {whole-body control, legged robots, challenging locomotion, and
	terrain mapping},
	pdftoolbar = true,
	colorlinks = true,
	linkcolor = black,
	citecolor = black,
	urlcolor = black,
}

\usepackage[usenames,dvipsnames]{xcolor}
\definecolor{blue_iit}{RGB}{51,51,255}
\usepackage{algpseudocode}
\usepackage{algorithm}
\usepackage{glossaries}
\usepackage[tight]{units}
\usepackage[normalem]{ulem} 
\usepackage{cancel}
\definecolor{Gray}{gray}{0.9}

\newacronym{imu}{IMU}{Inertial Measurement Unit}
\newacronym{rt}{RT}{Real Time}
\newacronym{cop}{CoP}{Center of Pressure}
\newacronym{ls}{LS}{Least Square}
\newacronym{pd}{PD}{Proportional-Derivative}
\newacronym{com}{CoM}{Center of Mass}
\newacronym{grfs}{GRFs}{Ground Reaction Forces}
\newacronym{fk}{FK}{Forward Kinematic}
\newacronym{ik}{IK}{Inverse Kinematic}
\newacronym{ocp}{OCP}{Optimal Control Problem}
\newacronym{nlp}{NLP}{Nonlinear Programming}
\newacronym{lc}{LC}{Landing Controller}
\newacronym{rl}{RL}{Reinforcement Learning}
\newacronym{mpc}{MPC}{Model Predictive Control}
\newacronym{to}{TO}{Trajectory Optimization}
\newacronym{ddp}{DDP}{Differential Dynamic Program}
\newacronym{dofs}{DoFs}{Degrees of Freedom}
\newacronym{ode}{ODE}{Ordinary Differential Equation}
\newacronym{msd}{MSD}{Mass-Spring-Damper}
\newacronym{pwbc}{pWBC}{projection-based Whole Body Control}
\newacronym{wbc}{WBC}{Whole Body Control}
\newacronym{sp}{SP}{Support Polygon}
\newacronym{nn}{NN}{Neural Network}
\newacronym{drl}{Deep-RL}{Deep Reinforcement Learning}
\newacronym{ppo}{PPO}{Proximal Policy Optimization}
\newacronym{td3}{TD3}{Twin Delayed DDPG}
\newacronym{grl}{GRL}{Guided Reinforcement Learning}
\newacronym{e2e}{E2E}{End-to-End Reinforcement Learning}
\newacronym{uarm}{UARM}{Uniformly Accelerated Rectilinear Motion}
\newacronym{cpg}{CPG}{Central Pattern Generator}

\newcommand{\Rnum}{\mathbb{R}} 

\newcommand{\vect}[1]{\mathbf{#1}} 

\newcommand{\mat}[1]{\ensuremath{\begin{bmatrix}#1\end{bmatrix}}}	

\newcommand\BibTeX{{\rmfamily B\kern-.05em \textsc{i\kern-.025em b}\kern-.08em
T\kern-.1667em\lower.7ex\hbox{E}\kern-.125emX}}

\makeatletter
\newcounter{definition*}
\newenvironment{definition*}[1][htb]
{\renewcommand{\ALG@name}{Definition}
	\let\c@algocf\c@megaalgorithm
	\begin{algorithm*}[#1]%
	}{\end{algorithm*}}
\makeatother

\makeatletter
\newcounter{definition}

\makeatother

\definecolor{sfahmi_blue}{RGB}{0.19,0.51,0.74}
\definecolor{LightBlue}{RGB}{0.4,0.4,1}

\newtheorem{problem}{Problem Statement}

\newtheorem{theorem}{Theorem}[section]

\newtheorem{lemma}[theorem]{Lemma}

\begin{document}
\title{Guided Reinforcement Learning for Omnidirectional  3D Jumping in Quadruped Robots}

\ifdefined\springer
	\author[1,2]{\fnm{Riccardo} \sur{Bussola}}\email{r.bussola@all3.com}
	\equalcont{These authors contributed equally to this work.}
	\author[1,29]{\fnm{Michele} \sur{Focchi}}\email{michele.focchi@unitn.it}
	\equalcont{These authors contributed equally to this work.}
	\author*[2]{\fnm{Giulio} \sur{Turrisi}}\email{giulio.turrisi@iit.it}
	\author[2]{\fnm{Claudio} \sur{Semini}}\email{claudio.semini@iit.it}
	\author[1]{\fnm{Luigi} \sur{Palopoli}}\email{luigi.palopoli@unitn.it}
	\affil*[1]{\orgdiv{Dipartimento di Ingegneria and Scienza dell'Informazione (DISI)}, \orgname{University of Trento}, \orgaddress{\street{Via Sommarive 9}, \city{Trento}, \postcode{38123},   \country{Italy}}}
	\affil[2]{\orgdiv{Dynamic Legged System (DLS)}, \orgname{Istituto Italiano di Tecnologia (IIT)}, \orgaddress{\street{Via San Quirico 19d}, \city{Genova}, \postcode{16163},  \country{Italy}}}
	\raggedbottom
\else
	\author{Riccardo Bussola$^{1,2}$, Michele Focchi$^{1,2}$, Giulio Turrisi$^{2}$, Claudio Semini$^{2}$, Luigi Palopoli$^{1}$
			\thanks{$^1$ The authors are with the Dipartimento di Ingegneria and Scienza dell'Informazione (DISI), University of Trento. Email:  \href{mailto:name.surname@unitn.it}{name.surname@unitn.it}}
			\thanks{$^2$ The authors are with Dynamic Legged System (DLS), Istituto Italiano di Tecnologia (IIT), Email:  \href{mailto:name.surname@iit.it}{name.surname@iit.it}. } }
\fi

\ifdefined\springer
\abstract{Jumping poses a significant challenge for quadruped robots, despite being crucial for many operational scenarios. While optimisation methods exist for controlling such motions, they are often time-consuming and demand extensive knowledge of robot and terrain parameters, making them less robust in real-world scenarios. Reinforcement learning (RL) is emerging as a viable alternative, yet conventional end-to-end approaches lack efficiency in terms of sample complexity, requiring extensive training in simulations, and predictability of the final motion, which makes it difficult to certify the safety of the final motion. To overcome these limitations, this paper introduces a novel guided reinforcement learning approach that leverages physical intuition for efficient and explainable jumping, by combining Bézier curves with a Uniformly Accelerated Rectilinear Motion (UARM) model. Extensive simulation and experimental results clearly demonstrate the advantages of our approach over existing alternatives.
}	
\else
\begin{abstract}
Jumping poses a significant challenge for quadruped robots, despite being crucial for many operational scenarios. While optimisation methods exist for controlling such motions, they are often time-consuming and demand extensive knowledge of robot and terrain parameters, making them less robust in real-world scenarios. Reinforcement learning (RL) is emerging as a viable alternative, yet conventional end-to-end approaches lack efficiency in terms of sample complexity, requiring extensive training in simulations, and predictability of the final motion, which makes it difficult to certify the safety of the final motion. To overcome these limitations, this paper introduces a novel guided reinforcement learning approach that leverages physical intuition for efficient and explainable jumping, by combining Bézier curves with a Uniformly Accelerated Rectilinear Motion (UARM) model. Extensive simulation and experimental results clearly demonstrate the advantages of our approach over existing alternatives.
\end{abstract}
\fi

\ifdefined\springer
\else
\fi

\maketitle

\textbf{Note:} A supplementary video can be found at:
\href{https://youtu.be/aSX_AapESPY?si=dFkFSgbSwFYV5Bl9}{\texttt{www.youtube.com/watch?v=aSX\_AapESPY}}

The code associated with this work can be found at:
\href{https://github.com/mfocchi/guided-rl-jump.git}{https://github.com/mfocchi/guided-rl-jump.git}.

\section{Introduction}\label{sec:introduction}

Quadruped robots are increasingly popular for their ability to move in
very difficult terrains~\cite{amatucci2024vero}, often characterised
by irregular slopes, gaps, and obstacles. This unique ability opens up
a full range of possible missions, which are difficult or impossible
for conventional robotic solutions (e.g., wheeled mobile robots) to
execute. The flexibility of legged robots comes with a price: the
complexity of their dynamic structure demands sophisticated solutions
for motion control.

While controlling legged locomotion~\cite{quadruped-survey} is a
relatively well-explored problem for walking and running
manoeuvres~\cite{tamols22}, the difficulty of the problems grows
significantly when the robot has to execute jumps longer than its legs
size. This type of manoeuvre is often required in scenarios such as the
exploration of hostile environments or for search and rescue in
post-disaster management, where quadruped robots have a
significant advantage over other platforms.

Unlike walking or running, which are based on continuous interaction
between the feet and the ground and offer ample opportunities for frequent
adjustments, jumping requires very precise control during the thrust
phase to achieve the desired trajectory, since the robot cannot
actively modify its trajectory during the subsequent flight phase. The
final phase, landing, demands stability and the ability to absorb
impact forces averting damages to the robot's structure and preventing
tipping over \cite{roscia2023reactive}.

From a control perspective, the key challenges of jumping can be summarised in the following points:
\begin{enumerate}
    \item \textbf{Optimisation of the thrust trajectory:} the desired lift-off state has to be reached while meeting dynamic constraints (e.g., joint position and velocity limits, and torque limitations).
    \item \textbf{High-precision:} When the robot is airborne, its flight is governed solely by the conservation of momentum, as the robot is acted upon only by external forces. Therefore, the final landing point and configuration can deviate significantly in the case of even small trajectory tracking errors during the thrust phase.
    \item \textbf{Strong real-time constraints:} The controller has to be executed in real-time as inputs need to be generated continuously and frequently during the entire jump, which is of limited duration.
\end{enumerate}

These requirements, coupled with the complex nonlinear dynamics of the quadruped robots and the unmodelled effects of the interaction between feet and the ground, make control design a very difficult task.

\subsection{Related work}

The challenging problem of quadruped jumping has attracted a significant attention
in the past few years, with many remarkable proposals in the literature.


A first group of researchers advocate the use of optimisation methods. A common solution to
manage the complexity of the task is to employ 
reduced models,
such as a 5-\gls{dofs} representation of the quadruped's side view
within the sagittal plane. This idea is often complemented by heuristic approaches
that utilise physical intuition in the design of controllers or
planners \cite{park17, roy20}. While these hand-crafted motion strategies
are grounded in physical reasoning, they often lack guarantees of
physical feasibility. Another common approach involves whole-body
numerical optimisation \cite{nguyen2021contact}. This method has led to
spectacular results, such as the complex aerial manoeuvres performed by the MIT Mini
Cheetah \cite{katz2019mini}. These manoeuvres include jumps, spins, flips, and barrel rolls,
executed through a
centroidal momentum-based nonlinear
optimisation \cite{chignoli2021online, garcia21, chignoli22}.

Despite the significance of these achievements, optimisation approaches
are subject to well-known limitations. First, solving the
nonlinear and nonconvex problem entails
a very high computational cost, making this family of solutions
difficult to deploy in real-time applications. Second, these approaches
require detailed system models, which are challenging to construct
on-the-fly in unstructured environments. The difficulty of accurately
capturing dynamic contact timing, combined with the sparse nature of the control
task (the controller operates only during the thrust phase), adds much to the complexity of the problem.
Given these challenges, the application of traditional optimisation
techniques \cite{chignoli2021online, garcia21, chignoli22} based on
offline planning proves to be impractical, particularly for reactive
approaches that demand frequent online replanning.
An interesting way of addressing at least part of these problems can be
found in the work of Song et al. \cite{song2022}, who propose an
optimal motion planner that generates a variety of energy-optimal
jumps. This is achieved by leveraging a library of pre-planned
heuristic motions, which are automatically selected based on
user-specified parameters or perception data. A different approach is
the one proposed by Li et al. \cite{cafe_mpc}, who achieve barrel
rolls using an \gls{mpc} approach. This method gradually relaxes the
planning problem along the prediction horizon—reducing model accuracy,
coarsening time steps, and loosening constraints—to improve
computational efficiency and performance. Ding et
al. \cite{ding2024robust} leverage a SLIP template model to create 3D
jumping motions.

In summary, in order to make \gls{mpc} and optimisation-based
methods computationally affordable, researchers frequently have to
compromise on optimality, imposing artificial
constraints~\cite{mastalli22, nguyen2021contact}, such as pre-fixed
contact sequences, prefixed flight-phase duration, or offline-optimised timing
parameters. The price to pay can be high in terms of robustness:
pre-specified foot contacts are known to create stability problems
for a large mismatch between the expected and actual contact
points \cite{li2024cafempc}. The problem is being
investigated. Bellegarda et al. \cite{bellegarda2024frog} take
inspiration from jumping frogs to reduce the number of parameters to a
small set which can be optimised online to avoid the sim-to-real
gap.

\begin{figure*}[tp!]
    \centering{
    \includegraphics[width=0.98\textwidth]{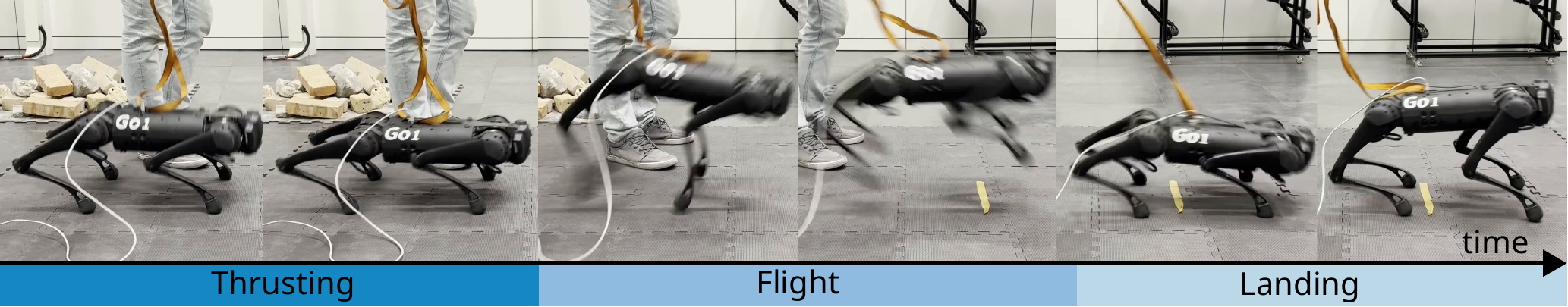}
    }
    \caption{Snapshots of Unitree Go1 robot executing the three stages of a jump: Trusting, Flight, and Landing phases.}
    \label{fig:jump_tax}
\end{figure*}

Advances in computational power, time-continuous action algorithms,
and highly parallelised simulators have spurred interest in
applying \gls{rl} to robot locomotion tasks. The pioneering work by
Lillicrap et al. \cite{lillicrap2015ContinuousCW} demonstrated that
actor-critic mechanisms combined with deep-Q networks can successfully
learn policies in continuous action domains. Building on this
fundamental result, \gls{rl} has been applied to various quadruped
locomotion tasks \cite{gehring2016practice, hwangbo2019learning,
peng2020, ji2022concurrent, rudin2021}, including some early attempts
at in-place hopping with
legs \cite{fankhauser2013reinforcement}. Innovative parkour
strategies, such as those proposed in \cite{hoeller2024anymal}, aim to
extend the application of~\gls{rl} towards more dynamic motion
patterns using purely~\gls{e2e} approaches.

\gls{e2e} \gls{rl} methods hold the promise to enable robotic
jumping, but they typically require millions of training steps to
converge on robust policies \cite{openAIbenchmark}. The primary
difficulty stems from the sparse and non-smooth nature of the
rewards. For jumping tasks, rewards are typically only provided after
landing, and abrupt contact transitions create sharp discontinuities
in the reward signal \cite{e2e_jumping, majid2021DeepRL}. The
resulting learning process can turn out to be unstable and
sample-inefficient. Additionally, \gls{rl} policies often lack
interpretability, complicating safety validation and making it harder
to ensure physical feasibility.

Several approaches address the issue of efficiency and robustness
in \gls{rl} training. For instance, combining trajectory
optimisation (TO) with \gls{rl} has been explored as a way to
bootstrap policy learning and overcome issues of local
optimality \cite{bogdanovic2022model, bellegarda2023robust,
grandesso2023cacto}. In these methods, \gls{rl} is used to generate
initial trajectories, thereby bootstrapping TO's exploration process,
avoiding local optima, and improving convergence. 

The action space
design is another crucial factor in \gls{rl} efficiency and
robustness \cite{vanderpanne2017, bellegarda19}. Some strategies
consider using position or torque references \cite{chen2023learning}
in joint \cite{aractingi2023controlling} or Cartesian space to enhance
learning. Shaffie et al. \cite{shafiee2024manyquadrupeds} further
generalised \gls{rl} policies across different quadruped
morphologies by formulating both actions and observations in task
space.
%

In order to achieve  effective RL-based jumping strategies, Yang et
al. \cite{yang2023continuous} warm-start the policy by learning
residual actions on top of a controller. This technique smooths the
reward signal and improves convergence, but is limited to
fixed-duration jumps. Vezzi et al. \cite{vezzi2024two} employ a
multi-stage learning approach, focusing on maximising jump height and
distance, though with limited consideration for landing accuracy.
Atanassov et al. \cite{atanassov2024curriculumbased} achieved
remarkable results with a complex multi-stage training process and
heavily engineered reward shaping. Although their work represents the
state-of-the-art in \gls{rl}-based jumping, the complexity of the
approach may limit its general applicability. For this reason, we will
use their work as a baseline to compare with the method proposed in
this paper.

A comprehensive overview by Esser et
al. \cite{esser2022guided} suggests that incorporating task-specific
knowledge (known as \gls{grl}) is the most effective way to address
the challenges posed by \gls{e2e} in complex tasks.  Along this
line, \cite{shafiee2024manyquadrupeds} injects some domain knowledge
into the training process, demonstrating that locomotion can be achieved
by modulating an underlying \gls{cpg}-based motor-control scheme.   In
conclusion, while optimisation-based and \gls{rl} methods have made
significant progress in quadruped locomotion, there is still a clear
gap in effective solutions for the quadruped jumping task,
particularly in non-flat environments.

\subsection{Contribution}
This work proposes a \gls{grl}-based strategy for precise \textbf{3D
omnidirectional quadruped jumping}. It is designed to plan the entire
jumping action from standstill to elevated (or depressed) terrain in
a few milliseconds, thereby enabling \textbf{real-time
applications}. This is achieved by parametrising the jump trajectory
in Cartesian space. This way we simplify the problem by modelling the jumping
task as a \textbf{single action}, which addresses the sparsity
inherent in conventional \gls{rl} approaches.

The core idea is to exploit physical intuitions to model the jump
trajectory. Our previous work on monopod
jumping \cite{bussola2024efficient} provided valuable insights into
using \gls{grl} to address the sparsity and complexity of linear
jumping, forming the foundation for our quadruped jumping
approach. That work employed Bezier curves to parametrise the
monopod's \gls{com} motion, concentrating learning within a small set
of parameters. In this study, we build upon these results for the more
complex quadruped scenario, introducing control over both the
translational and angular dynamics. The latter was not included in
the preparatory work on monopods. 

Our simulation and experimental results show that this approach not
only enhances accuracy and secures quick, robust convergence of the
learning process, but also instils a high degree of reliability into
the \gls{rl} policy, ensuring safer and more dependable actions. For instance, given
the lift-off configuration, we can utilise ballistic equations to make
tight predictions on the landing position and timing.
Unlike other methods requiring multi-stage learning or preliminary
goal-specific tasks, ours is a \textit{single-stage} learning process
that directly achieves the final goal without intermediate training
steps. This simplification improves sample efficiency and
significantly impacts training time.

From a functional perspective, the integration of orientation and
angular velocity allows us to exploit the robot's full dynamic
capabilities, optimising its orientation at take-off; 
additionally, controlling
angular velocity enables the robot to adjust its pitch mid-air,
reducing impact forces upon landing. Finally, this approach
facilitates novel jump manoeuvres, such as in-place twist jumps, where
the robot can rotate its body mid-air without linear displacement,
expanding its repertoire of dynamic actions.

From an implementation standpoint, the method computes \gls{com}
references, which are then mapped to joint positions
using \gls{ik}. Planning the action in Cartesian space also renders
the system adaptable to various quadruped morphologies (or even other
legged robot morphologies, e.g., humanoids). The joint references are
tracked by a low-level controller, leveraging the robustness of
classical control methods to handle uncertainties in friction, masses and
inertias. This alleviates issues arising from imprecise parameter
knowledge, removing the need for extensive domain randomisation during
training and making our method more robust and immune to overfitting.
To summarise, the main contributions of our work are:
\begin{itemize}
    \item A novel, sample-efficient \gls{rl} method designed for quadruped robot omnidirectional jumping, which leverages the concept of \textbf{Guided Reinforcement Learning} to inject system knowledge into the task formulation.
    \item A formulation of an action-space for jumping tasks, which is: 1) robot-agnostic; 2) provides the possibility to predict the outcome of the desired robot jump motion, enhancing the predictability of \gls{rl} in such scenarios; and 3) can be used as an effective heuristic for predicting the jump apex and touch-down moments.
    \item A large set of simulation and experiments on two different quadrupeds (Aliengo and Go1) validating the method's efficacy and repeatibility. The learned policy is transferred 
    from simulation to real robots, and its robustness is evaluated in simulation under varying damping and mass conditions. Compared to \cite{atanassov2024curriculumbased}, our approach achieves a lower standard deviation in landing error and higher sample efficiency. For both simulation and real experiments, we also provide a statistical analysis of the accuracy.
\end{itemize}
%
\subsection{Outline}
In Section \ref{sec:guided_rl} we illustrate an overview of the proposed method, in Section \ref{sec:traj_parametrization}
we define the adopted parametrization of the  thrust trajectory, we illustrate in Section \ref{sec:rl} the Learning  model used in our \gls{grl} approach. In Section \ref{sec:results} we report simulation and experimental results of omnidirectional jumps achieved with the Aliengo and Go1 robots, finally, in Section \ref{sec:conclusion} we draw the conclusions.

\nomenclature{$s$}{State of the system, representing the environment at a specific point in time}
\nomenclature{$a$}{Action taken by the agent, corresponding to a control input}
\nomenclature{$r$}{Reward signal, quantifying how well the agent's action aligns with the task's objectives}
\nomenclature{$\mathbf{c}$}{Position of the Center of Mass (COM) of the robot}
\nomenclature{$\mathbf{\Phi}$}{Orientation of the trunk, represented by Euler angles}
\nomenclature{$\mathbf{\dot{c}}$}{Linear velocity of the Center of Mass (COM)}
\nomenclature{$\mathbf{\dot{\Phi}}$}{Angular velocity of the trunk}
\nomenclature{$T_{th}$}{Total thrust time during the jump}
\nomenclature{$T_{th_{b}}$}{Thrust time for the Bézier position trajectory phase}
\nomenclature{$T_{th_{e}}$}{Thrust time for the UARM trajectory phase}
\nomenclature{$\mathbf{P}$}{Control points for the Bézier position trajectory}
\nomenclature{$\mathbf{Q}$}{Control points for the Bézier orientation trajectory}
\nomenclature{$\mathbf{c}_{lo_{b}}$}{Bézier lift-off position of the robot's Center of Mass}
\nomenclature{$\mathbf{\dot{c}}_{lo_{b}}$}{Bézier lift-off linear velocity of the robot's Center of Mass}
\nomenclature{$\mathbf{c}_{lo_{e}}$}{Explosive lift-off position of the robot's Center of Mass (UARM phase)}
\nomenclature{$\mathbf{\dot{c}}_{lo_{e}}$}{Explosive lift-off linear velocity of the robot's Center of Mass (UARM phase)}
\nomenclature{$\mathbf{\Phi}_{lo}$}{Lift-off orientation of the trunk}
\nomenclature{$\mathbf{\dot{\Phi}}_{lo}$}{Lift-off angular velocity of the trunk}
\nomenclature{$k$}{Velocity multiplier in the UARM trajectory parametrization}
\nomenclature{$d$}{Displacement in the UARM trajectory phase}
\nomenclature{$\mathbf{q}$}{Joint configuration}
\nomenclature{$\mathbf{\dot{q}}$}{Joint velocities}
\nomenclature{$\boldsymbol{\tau}$}{Joint torques}
\nomenclature{$\mathbf{F}$}{Contact forces at the feet}
\nomenclature{$\mathbf{J}$}{Jacobian matrix}
\nomenclature{$\mathbf{K}_{p}$}{Proportional gain in the control law}
\nomenclature{$\mathbf{K}_{d}$}{Derivative gain in the control law}
\nomenclature{$m$}{Mass of the robot}
\nomenclature{$W$}{World frame}
\nomenclature{$B$}{Base frame}

\section{Guided Reinforcement Learning}
\label{sec:guided_rl}
\subsection{Jump Taxonomy}
In this work, we have chosen to adopt a jumping technique known as pronking (or stotting), 
where all four legs push simultaneously on the ground, as seen in some mammals
\footnote{ The other jumping strategy is the leaping technique, in which the animal initially 
pushes the ground with all four legs. As it propels forward, it pitches its body up so that only the hind legs remain in contact
with the ground, and the final push is executed by these remaining two legs.}.
This creates a single continuous thrust phase, enabling straightforward parametrization.
Biomechanically, it consists of three phases: thrust, flight, and landing (see Fig. \ref{fig:jump_tax}). 
\textbf{Thrust Phase:}
it involves an initial compression followed by explosive decompression, akin to a spring release. An effective timing and an efficient energy management are crucial for optimising acceleration, flight duration, and landing. A proper force distribution prevents imbalances and ensures safe execution, considering terrain constraints such as friction. 
 
\textbf{Flight Phase}
Once airborne, the   motion follows ballistic principles and is governed solely by gravity. Although thrust cannot be adjusted mid-flight, minor leg repositioning aids obstacle avoidance. Additionally, retracting legs mid-flight and extending them before touchdown optimises impact absorption, as explained later.
During the flight phase, no additional thrust is generated, making the jump trajectory entirely dependent on momentum generated at lift-off. 

\textbf{Landing Phase}
In this phase, ground contact is re-established while mitigating impact forces. The leg positioning from the flight phase is key to absorbing shocks and maintaining stability. A compliant landing strategy is beneficial to dissipate kinetic energy, preventing tipping and structural stress, and ensuring a smooth transition back to standing. To achieve safe landing, Roscia et al. \cite{roscia2023reactive} developed an optimisation-based \gls{lc} using a Variable Height Springy Inverted Pendulum (VHSIP) model to manage flight and landing phases. The \gls{lc} adjusts foot positions based on horizontal velocity during flight, ensuring stable landings without bouncing and slippage. While an \gls{lc} significantly affects jump performance, we deliberately omitted it during both training and experiments to focus on the thrust phase—the most critical stage, as it determines the force needed to reach the target and ultimately dictates the jump's success.

\subsection{Method Overview}

Since the landing position depends solely on the \textit{lift-off} configuration,  selecting this configuration appropriately is crucial for reaching the desired target. In this section, we analyse how the lift-off state influences the landing position, emphasising the inherent complexity of jump planning.

To simplify the problem, the robot flight phase is modeled as the projectile motion of a rigid body influenced only by gravity. 
In the absence of other external forces, this \textit{ballistic} trajectory lies within the plane defined by the lift-off point and the target location. Consequently, planning a jump reduces to finding the inverse relationship 
between the desired target location $c_{tg}$ and the corresponding lift-off configuration—comprising position $c_{lo}$ and velocity $\dot{c}_{lo}$. Since the trajectory remains constrained to the plane tangential to both the initial and target positions, all feasible lift-off \gls{com} positions and velocity vectors that result in the same target $c_{tg}$, lie on this planar manifold. 
Any lift-off velocity with a component outside this plane would, due to momentum conservation during the flight, 
lead to a different landing point.

The projectile motion of the system can be described by the following 
set of equations: 
\begin{equation}
\left\{
\begin{aligned}
     \mathbf{c}_{tg,x} &= \mathbf{c}_{lo,x} + \mathbf{\dot{c}}_{lo,x} T_{fl} \\
     \mathbf{c}_{tg,y} &= \mathbf{c}_{lo,y} + \mathbf{\dot{c}}_{lo,y} T_{fl} \\
     \mathbf{c}_{tg,z} &= \mathbf{c}_{lo,z} + \mathbf{\dot{c}}_{lo,z} T_{fl} - \frac{1}{2} g T^{2}_{fl}
     \label{eq:ballistic_eq}
\end{aligned}
\right.
\end{equation}
where $T_{fl}$ represents the duration of the flight phase, and $g$ denotes the acceleration due to gravity.
For the sake of simplicity, in this description we consider the execution of a forward jump, neglecting orientation.
\begin{figure}[b]
\centering
    \includegraphics[width=0.8\columnwidth]{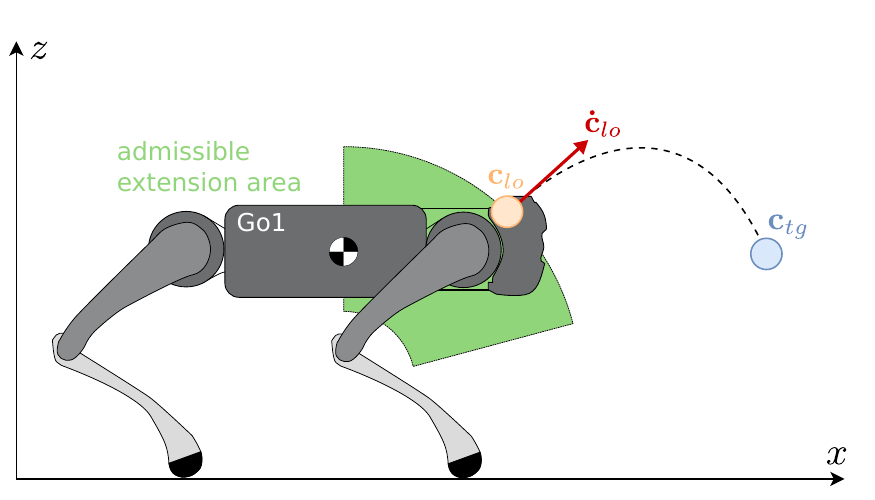}
    \caption{Side-view of the ballistic problem.}
    \label{fig:ballistic-problem}
\end{figure}
In view of the property just described, 
we can analyse the ballistic motion within a two-dimensional side-view space (see Fig. \ref{fig:ballistic-problem}), involving
the z-axis and x-axis (first and third equation in \eqref{eq:ballistic_eq}). 
To investigate the impact of the 
lift-off linear velocity on the jump outcome, given a desired target position,  we fix the lift-off position ($c_{lo,x}, c_{lo,z}$), and work out the  flight time from the first equation of \eqref{eq:ballistic_eq} and replace in the third. 
Then,  by simple algebraic manipulations,  the vertical component 
of the lift-off velocity $\dot{c}_{lo,z}$ can be expressed as 
a  function of the horizontal one $\dot{c}_{lo,x}$

\begin{equation}
    \vect{\dot{c}}_{lo,z}=\frac{\vect{c}_{lo,z}-\vect{c}_{tg,z}}{\vect{c}_{tg,x}-\vect{c}_{lo,x}}\vect{\dot{c}}_{lo,x}+\frac{\vect{c}_{tg,x}-\vect{c}_{lo,x}}{2}g\frac{1}{\vect{\dot{c}}_{lo,x}}
\label{eq:balistic_problem}
\end{equation}

This hyperbolic relation tells us that an infinite set of solution pairs can be chosen: for instance,  it is possible to jump with 
smaller horizontal velocity staying longer in the air (i.e. higher vertical velocity) or the opposite. 
Flight time decreases as horizontal lift-off velocity increases since a lower horizontal component requires greater vertical thrust, prolonging flight. Optimizing for minimal flight time requires maximising horizontal velocity; however, not all solutions are compatible with the robot's  physical limits, such as joint position, velocity, and torque limits, and with the constraints posed by the environment. 
Given these considerations, we can model the problem of generating the thrust phase in the following terms:
\begin{problem}
  \label{prob:form}
    Synthesize a thrust phase that produces a lift-off state (i.e. \gls{com} position and velocity) that: 1. satisfies~\eqref{eq:ballistic_eq},
    2. copes with the potentially adverse conditions posed by the environment (i.e. contact stability, friction constraints), 
    3. satisfies the physical and actuation constraints.
\end{problem}


Including lift-off position as a decision variable adds complexity, 
making the problem highly non-convex and requiring Nonlinear Programming (NLP) solvers.
This problem is typically addressed alongside trajectory planning during the thrust phase, but its complexity often requires artificial constraints like fixed thrust or flight times to reduce computational load. NLP solvers, while useful, are sensitive to initialization and prone to local minima, making them unsuitable for real-time jumping tasks. Instead, \gls{grl} offers an efficient alternative solution by treating the lift-off configuration as the primary action of our policy, as it is the crucial factor determining the overall jump outcome. 
We employ reinforcement learning to generate optimal thrust trajectories, which are executed by a low-level tracking controller. Since the trajectory can be parameterized and computed in advance, it enables fast execution through a lightweight neural network. During the jump, the low-level controller continuously tracks the planned motion, ensuring robustness and closed-loop stability.

Our approach is built on three key ideas.
 First, learning is conducted in Cartesian space, this design choice improves generalisation across different robot morphologies and configurations.
Second, because the system follows ballistic dynamics during the airborne phase, the final landing position can be predicted \textit{deterministically}.  This allows the learning process to focus exclusively on the thrust phase—the moment when the robot propels itself off the ground and  enables the integration of a  \textit{safety filter} (see Section \ref{sec:physical_check}) that eliminates infeasible actions through simple computations.
Third, inspired by biological systems, we guide the learning with prior knowledge about what a plausible motion should look like. 

In mammals, movement patterns are shaped by evolutionary priors; similarly, we embed physical intuition into the learning process. A jump begins with a “charging” motion, during which the legs compress, followed by an explosive extension that accelerates the \gls{com} upward and forward, generating momentum for lift-off. This charging motion aims to exploit the full joint range to produce acceleration.

Our key idea is to use a parametric curve to capture this motion, thereby enhancing and stabilising the learning process. Among the class of parametric curves suitable for this purpose, Bézier curves are particularly convenient due to their simple and efficient definition and computation. Notably, these curves offer two key advantages: their derivatives have simple expressions, and the curve lies entirely within the convex hull of its control points.

Using this representation, we apply an on-policy RL algorithm, \gls{ppo}~\cite{ppo}, to learn the Bézier parameters (actions) by minimising cost functions similar to those used in optimal control. 

Figure \ref{fig:framework} illustrates our \gls{rl} pipeline. The agent operates in 
either training or inference mode. During training, the Critic and Actor 
neural networks are periodically updated (as indicated by the dashed lines in the figure)
once enough data has been gathered.
\begin{figure}[tbp]
\centering
\includegraphics[width=1.0\columnwidth]{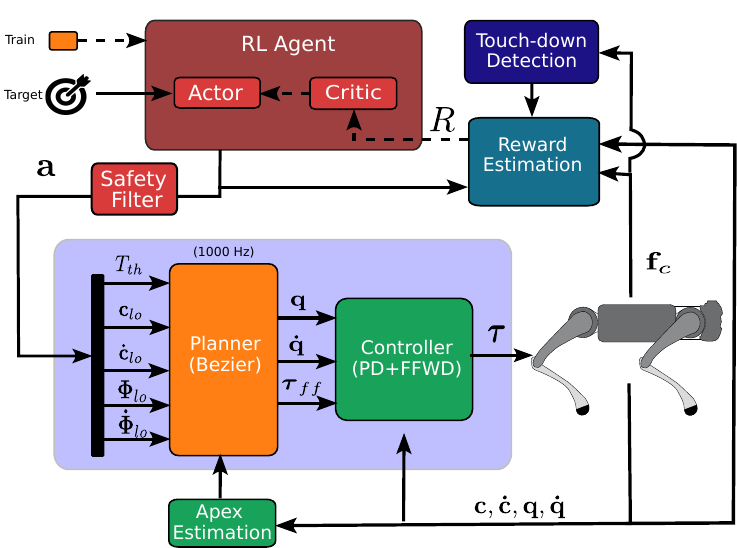}
\caption{\small{Diagram of the  \gls{grl} Framework. 
The framework is split into two levels: 
the \gls{rl} agent and the planner. The \gls{rl} agent produces an action for the planner 
based on a desired target. This computes a B\'ezier reference curve that is 
mapped into joint motion via inverse kinematics and tracked by 
the \gls{pd} controller that provides the joint torques to feed the robot.
During the training, at the end of each episode, a 
reward is computed and fed back to the \gls{rl} agent. 
Dashed lines are active when the framework is in training mode. }}
\label{fig:framework}
\end{figure}
\subsection{Control}
To track a given \gls{com} trajectory, we employ a kinematics-based approach
combined with gravity compensation.  
\gls{ik} is employed to compute the joint configurations and velocities that correspond to the given  \gls{com}  trajectory.
The computed joint positions $ \mathbf{q}^{d}\in \Rnum^n $ and velocities $ \mathbf{\dot{q}}^{d}\in \Rnum^n $ are then passed to a 
\gls{pd} controller, along with a feedforward torque $ \boldsymbol{\tau}_{ff}\in \Rnum^n$ derived from gravity compensation.

\begin{equation}
\begin{aligned}
    \boldsymbol{\tau} &= \mathbf{K}_{p}(\mathbf{q}^{d} - \mathbf{q}) + \mathbf{K}_{d}(\mathbf{\dot{q}}^{d} - \mathbf{\dot{q}}) + \boldsymbol{\tau_{ff}} \\
    \boldsymbol{\tau}_{ff} &= -\mathbf{J}_{cj}^{T}\underbrace{(\mathbf{J}_{cb}^{T})^{\dagger} \mathbf{w}^{g}}_{\mathbf{f}_c}
\end{aligned}
\end{equation}
where $-\mathbf{J}_{cj}^{T}(\mathbf{J}_{cb}^{T})^{\dagger}$ is the mapping of the gravity wrench $\mathbf{w}^{g} \in \Rnum^6$ into joint torques and $\vect{f}_c \in \Rnum^{3c}$ is the correspondent vector of feet forces, $n=12$ is the number of actuated \gls{dofs} and $c=4$ the number of contacts. We made the choice of using a sub-optimal simple control scheme instead 
of more elaborate ones to show the capability of the learning framework to compensate for controller inaccuracies.  

\section{Trajectory parametrization}
\label{sec:traj_parametrization}
\subsection{Bezier Curves}
B\'ezier curves are defined by a set of control points $\mathbf{P}$ 
and constructed using Bernstein basis polynomials. The general formulation of a Bézier curve of degree $n$ is as follows:
\begin{equation}
\label{eq:bez}
    \begin{aligned}
        \mathbf{B}(t) &= \sum_{i=0}^{n}\binom{n}{i}t^{i}(1-t)^{n-i}\mathbf{P}_{i} \;\;\;\; 0 \leq t \leq 1\\
        &= \sum_{i=0}^{n}\frac{n!}{i!(n-i)!}t^{i}(1-t)^{n-i}\mathbf{P}_{i} \\
        &= \sum_{i=0}^{n}b_{i}^n(t)\mathbf{P}_{i}
    \end{aligned}
\end{equation}
where $n$ is the order of the curve, $b_{i}^n$ is the Bernstein basis polynomial of degree $n$ and $\mathbf{P}_{i}\in \Rnum^3$ is the $i^{th}$ control point. For a Bézier Curve of order $n$ the number $\#{\mathbf{P}}$ of control points  is equal to
\begin{equation}
    \#{\mathbf{P}}= n+1
\end{equation}
One of the most valuable properties of the Bézier curve is its differentiability, with its derivative also being a Bézier curve of order $n - 1$ (see Appendix \ref{sec:bez_der} for details on the derivation) 
\begin{equation}
    \label{eq:bez-der}
    \mathbf{\dot{B}}(t) = \sum_{i=0}^{n-1}b_{i}^{n-1}(t)\mathbf{P^{\prime}}_{i} \;\;\;\; 0 \leq t \leq 1
\end{equation}
with control points defined as
\begin{equation}
    \mathbf{P^{\prime}}_{i} = n\left(\mathbf{P}_{i+1}-\mathbf{P}_{i}\right)
\end{equation}
The curve is originally defined over a normalized time interval $0 \leq t \leq 1$ but the formulation can be defined over arbitrary time intervals  $ [0, T_{th}] $
with the following generalized form
\begin{equation}
    \mathbf{B}(t) = \sum_{i=0}^{n}b_{i}^n\left(\frac{t}{T_{th}}\right)\mathbf{P}_{i}  \;\;\;\; 0 \leq t \leq T_{th}
\end{equation}
It follows that, once the control points of a B\'ezier curve of order $n$ are known, the control points of
its derivative curve can be directly obtained
\begin{equation}
    \begin{aligned}
        \mathbf{P^{\prime}}_{i} &= \frac{n}{T_{th}}\left(\mathbf{P}_{i+1}-\mathbf{P}_{i}\right) \\
        \mathbf{\dot{B}}(t) &= \sum_{i=0}^{n-1}b_{i}^{n-1}\left(\frac{t}{T_{th}}\right) \mathbf{P^{\prime}}_{i} \;\;\;\; 0 \leq t \leq T_{th}
    \end{aligned}
\end{equation}
When selecting the order of the Bézier curve for our problem, we must consider the need to model both position and velocity. For any Bézier curve, the first and last control points correspond to the initial and final values of the curve. A third-order Bézier curve offers an ideal balance for modeling the position trajectory, as it provides four control points. The two intermediate control points can be used to shape the velocity profile, since all the four  control points are later used to compute the derivative, and thereby the velocity trajectory. For the explicit form of  the cubic B\'ezier curve and its derivative see Appendix \ref{sec:explicit}.
\subsection{Jumping Strategy for quadrupeds}
This work builds upon our previous work for a monopod \cite{bussola2024efficient} by including the angular dynamics, 
which has an impact on the distribution of contact forces, balance, and slippage,  at the same time 
enabling dynamic maneuvers such as in-place twists for rapid reorientation.
%

The  robot configuration is represented by the pair of vectors $(\mathbf{s}, \mathbf{\dot{s}})$, where $\mathbf{s}$ includes the \gls{com} position $\mathbf{c}$ and trunk orientation $\boldsymbol{\Phi}$ (expressed in ZYX Euler angles\footnote{It is reasonable to use Euler angle parametrization for the orientation during the thrust phase, despite the well known gimbal lock singularity issue at $\theta=\pi/2$ because we are not expecting such big orientation changes.}), and $\mathbf{\dot{s}}$ includes the \gls{com} linear velocity $\mathbf{\dot{c}}$ and trunk Euler Rates $\boldsymbol{\dot{\Phi}}$. 
Additionally, we define the robot target configuration $\mathbf{s}_{tg}$ 
as the vector containing the desired \gls{com}  position 
$\mathbf{c}_{tg}$ and trunk orientation $\boldsymbol{\Phi}_{tg}$ at \textit{landing}.
Starting from the explicit forms \eqref{eq:expl-bez} of both the cubic Bézier curve  and its quadratic derivative, we expressed the fulfillment of the boundary conditions as a linear system \eqref{eq:bezier-control-points}. Once solved, this system provides the control points of the thrust trajectory which links the initial state $(\mathbf{s}_{0}, \mathbf{\dot{s}}_{0})$, to the lift-off state $(\mathbf{s}_{lo}, \mathbf{\dot{s}}_{lo})$:
%
%
\begin{equation}
\resizebox{\columnwidth}{!}{
    $
    \begin{cases}
    \begin{array}{l}
    \mathbf{P}_{0} = \mathbf{c}_{0} \\
    \mathbf{P}_{3} = \mathbf{c}_{lo} \\
    \mathbf{P'}_{0} = \frac{3}{T_{th}}(\mathbf{P}_{1}-\mathbf{P}_{0}) = \dot{\mathbf{c}}_{0} \\
    \mathbf{P'}_{1} = \frac{3}{T_{th}}(\mathbf{P}_{2}-\mathbf{P}_{1}) \\
    \mathbf{P'}_{2} = \frac{3}{T_{th}}(\mathbf{P}_{3}-\mathbf{P}_{2}) = \dot{\mathbf{c}}_{lo}
    \end{array}
    \end{cases}
    \begin{cases}
    \begin{array}{l}
    \mathbf{Q}_{0} = \mathbf{\Phi}_{0} \\
    \mathbf{Q}_{3} = \mathbf{\Phi}_{lo} \\
    \mathbf{Q'}_{0} = \frac{3}{T_{th}}(\mathbf{Q}_{1}-\mathbf{Q}_{0}) = \dot{\mathbf{\Phi}}_{0} \\
    \mathbf{Q'}_{1} = \frac{3}{T_{th}}(\mathbf{Q}_{2}-\mathbf{Q}_{1}) \\
    \mathbf{Q'}_{2} = \frac{3}{T_{th}}(\mathbf{Q}_{3}-\mathbf{Q}_{2}) = \dot{\mathbf{\Phi}}_{lo}
    \end{array}
    \end{cases}
   $
   \label{eq:bezier-control-points} }%
\end{equation}
The only shared variable between the orientation and position trajectories is the thrusting time $T_{th}$, due to the duration constraints that both parameterizations must respect. Beyond this, the two trajectories remain completely independent. 
Including an angular component enables regulation of the angular velocity at lift-off, making it possible to achieve landing postures that are more optimal in terms of \gls{grfs} distribution.
In Fig. \ref{fig:bez-orientation}, the 3D visualization of a thrust trajectory for a simple forward jump is shown, both for the pure position parameterization and for the combined position and orientation parameterization, with a pitch derivative at lift-off of 4 rad/s. 
\begin{figure}[htbp]
     \centering
     \includegraphics[width=0.49\columnwidth]{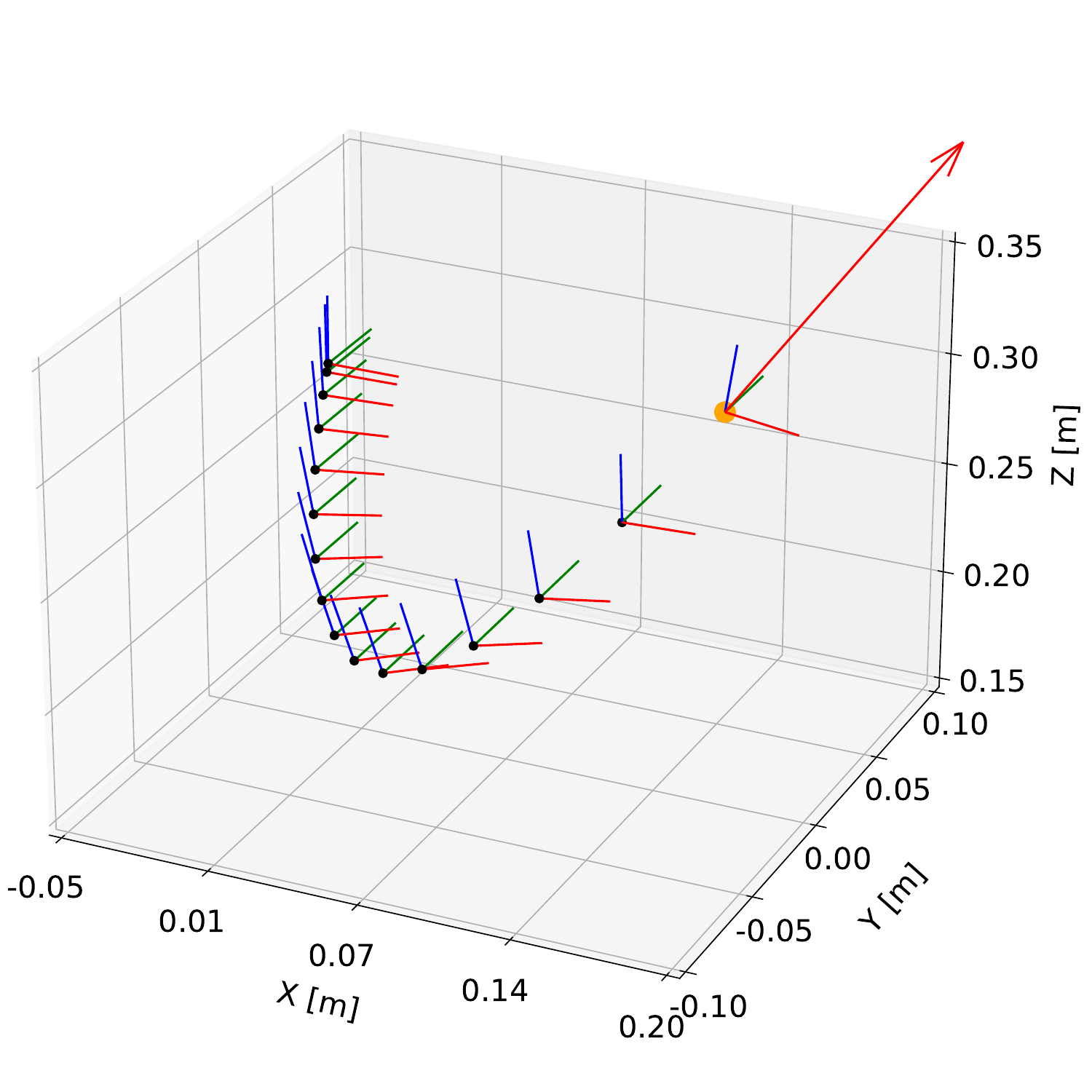}
     \includegraphics[width=0.49\columnwidth]{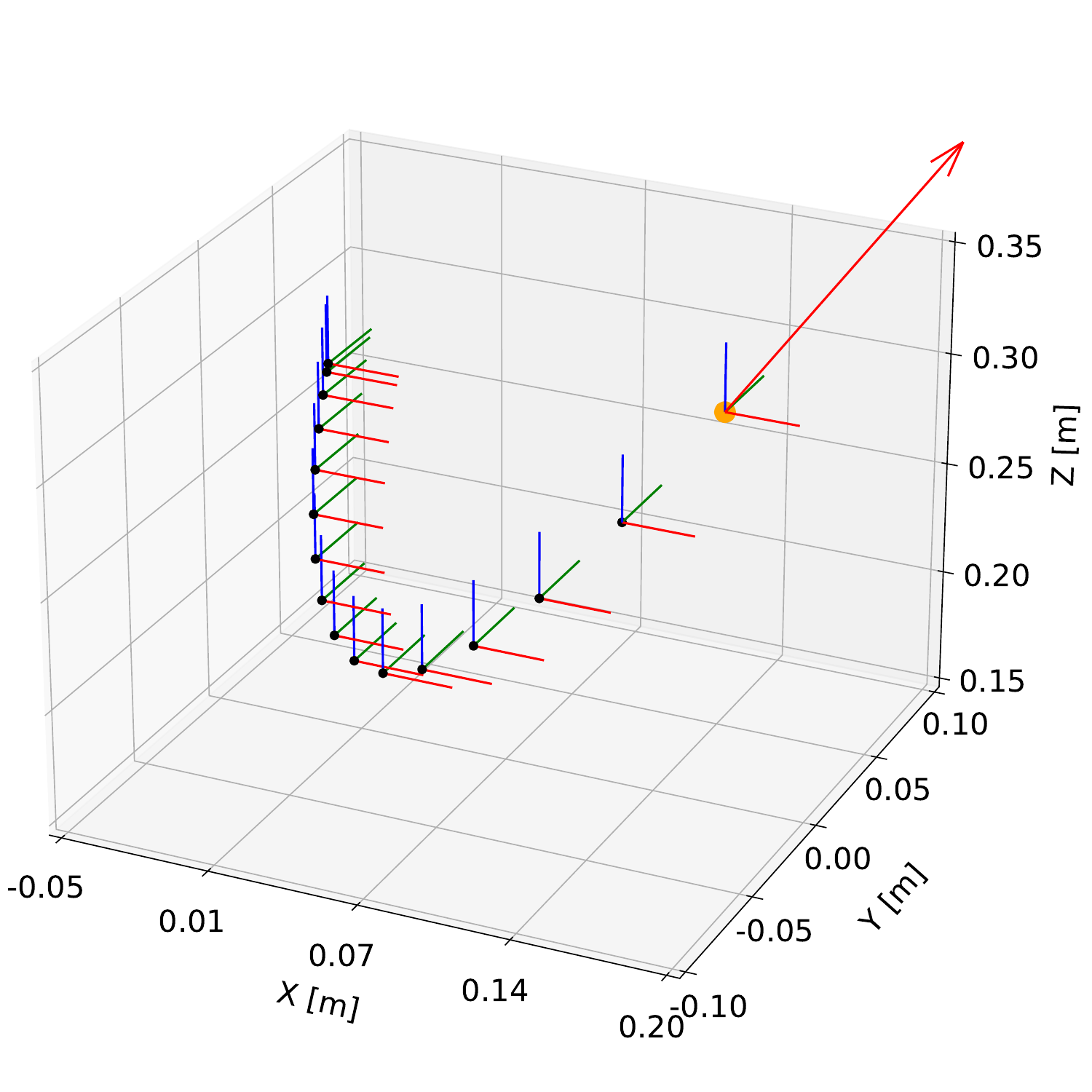}
     \caption{3D visualization of the thrust trajectory with (left) and without (right) angular motion.}
     \label{fig:bez-orientation}
\end{figure}

\subsection{Explosive Thrust}

The Bézier-based parameterization  \eqref{eq:expl-bez} might limit the expressiveness of the thrust, particularly during the explosive decompression phase.  According to the control point calculations in \eqref{eq:bezier-control-points}, for a given fixed thrust time, if we increase the lift-off velocity while keeping other boundary conditions unchanged, the control point $\mathbf{P}_{2}$ stretches the curve, resulting in a longer path to travel within the same timeframe.
Certain lift-off velocities may need to be excluded from the feasible set, as the resulting trajectory could violate the robot's workspace constraints—for example, causing excessive squatting that would require an unfeasibly low height (i.e., the robot's belly would touch the ground, see Fig. \ref{fig:traj-expl-comp}). A promising solution to this issue is to split the trajectory in two parts where after the initial decompression stage, dictated by the Bezier, the path, in a second part,  flattens to a straight line where the robot follows an \gls{uarm} (see Fig. \ref{fig:traj-expl}). 
\begin{figure}[ht!]
\centering
    \includegraphics[width=0.7\columnwidth]{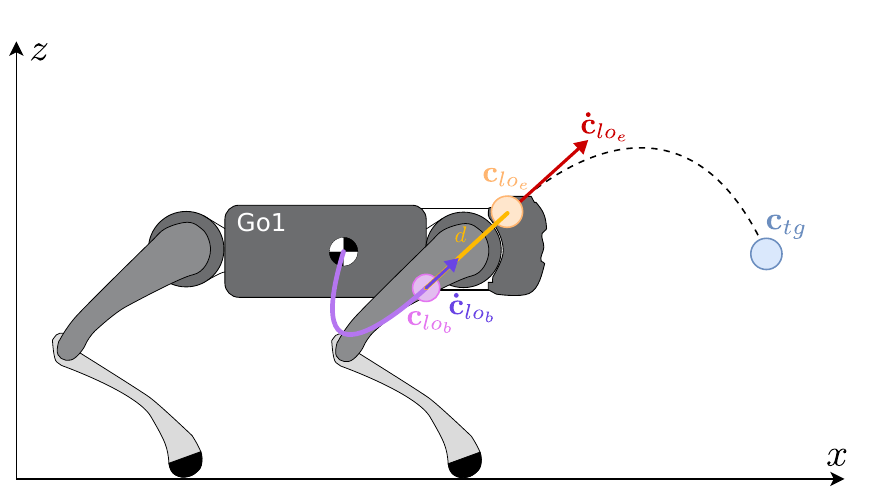}
    \caption{Thrust trajectory with the new Bézier-\gls{uarm} parametrization.}
    \label{fig:traj-expl}
\end{figure}
This would allow us to maintain control over the thrust trajectory shape
(i.e. do not result in workspace violations or unfeasible robot heights) while ensuring  higher lift-off velocities. 
As explained later, this solution will introduce only a few additional parameters to the \textit{linear} part of the thrust trajectory, 
while the orientation remains unchanged.
Let's  outline the \gls{uarm} equations of motion:
\begin{equation}
\left\{
    \begin{aligned}
        & v_{f} = v_{0} + a(t-t_{0})\\[2mm]
        & d_{f} = \frac{1}{2}a(t-t_{0})^{2} + v_{0}(t-t_{0}) + d_{0}
    \end{aligned}
\right.
\end{equation}
Where \(v_{0}\) and \(v_{f}\) are the magnitudes of the initial and final velocities, \(d_{0}\) and \(d_{f}\) are the initial and final displacements (on a line), and \(t_{0}\) is the initial time. 
Solving  for the time variable, setting $t_{0} = 0$ for the sake of simplicity,  after some manipulation we obtain the time term $t$ and the acceleration term $a$:
\begin{equation}
\left\{
    \begin{aligned}
        & a = \frac{1}{2} \frac{v_{f}^{2} - v_{0}^{2}}{d_{f} - d_{0}} \\[2mm]
        & t = \frac{v_{f} - v_{0}}{a}
    \end{aligned}
\right.
\label{eq:uarm-system}
\end{equation}
Previously, our thrust trajectory was characterized by the initial and lift-off \gls{com} positions and linear velocities $(\mathbf{c}_{0}, \mathbf{\dot{c}}_{0})$, $(\mathbf{c}_{lo}, \mathbf{\dot{c}}_{lo})$. Now, we introduce an intermediate configuration, \((\mathbf{c}_{lo_{b}}, \mathbf{\dot{c}}_{lo_{b}})\), which represents the endpoint of the Bézier parameterization and the starting point for the \gls{uarm} trajectory. The actual lift-off configuration, will be the final condition of the \gls{uarm} trajectory, now denoted as \((\mathbf{c}_{lo_{e}}, \mathbf{\dot{c}}_{lo_{e}})\). 
Given this new concept and notation, the system \eqref{eq:uarm-system} becomes:
\begin{equation}
\left\{
    \begin{aligned}
        & a = \frac{1}{2} \frac{\Vert\mathbf{ \dot{c}}_{lo_{e}}\Vert^{2} - \Vert\mathbf{\dot{c}}_{lo_{b}}\Vert^{2}}{\Vert\mathbf{c}_{lo_{e}}\Vert - \Vert\mathbf{c}_{lo_{b}}\Vert} \\
        & T_{th_{e}} = \frac{\Vert\mathbf{\dot{c}}_{lo_{e}}\Vert - \Vert\mathbf{\dot{c}}_{lo_{b}}\Vert}{a}
    \end{aligned}
\right.
\label{eq:uarm-system-new}
\end{equation}
Where \(T_{th_{e}}\) is the duration of  the \gls{uarm} part of the thrust trajectory. 
Renaming the duration of Bézier curve as \(T_{th_{b}}\), so the total thrust time becomes:
\begin{equation}
    T_{th} = T_{th_{b}} + T_{th_{e}}
    \label{eq:tot-time}
\end{equation}
Introducing the explosive lift-off configuration as additional parameters would involve adding six new variables, as \(\mathbf{\dot{c}}_{lo_{e}} \in \mathbb{R}^{3}\) and \(\mathbf{c}_{lo_{e}} \in \mathbb{R}^{3}\). However, such an increase would be inconveniently large.
\begin{figure}[ht!]
\centering
    \includegraphics[width=0.8\columnwidth]{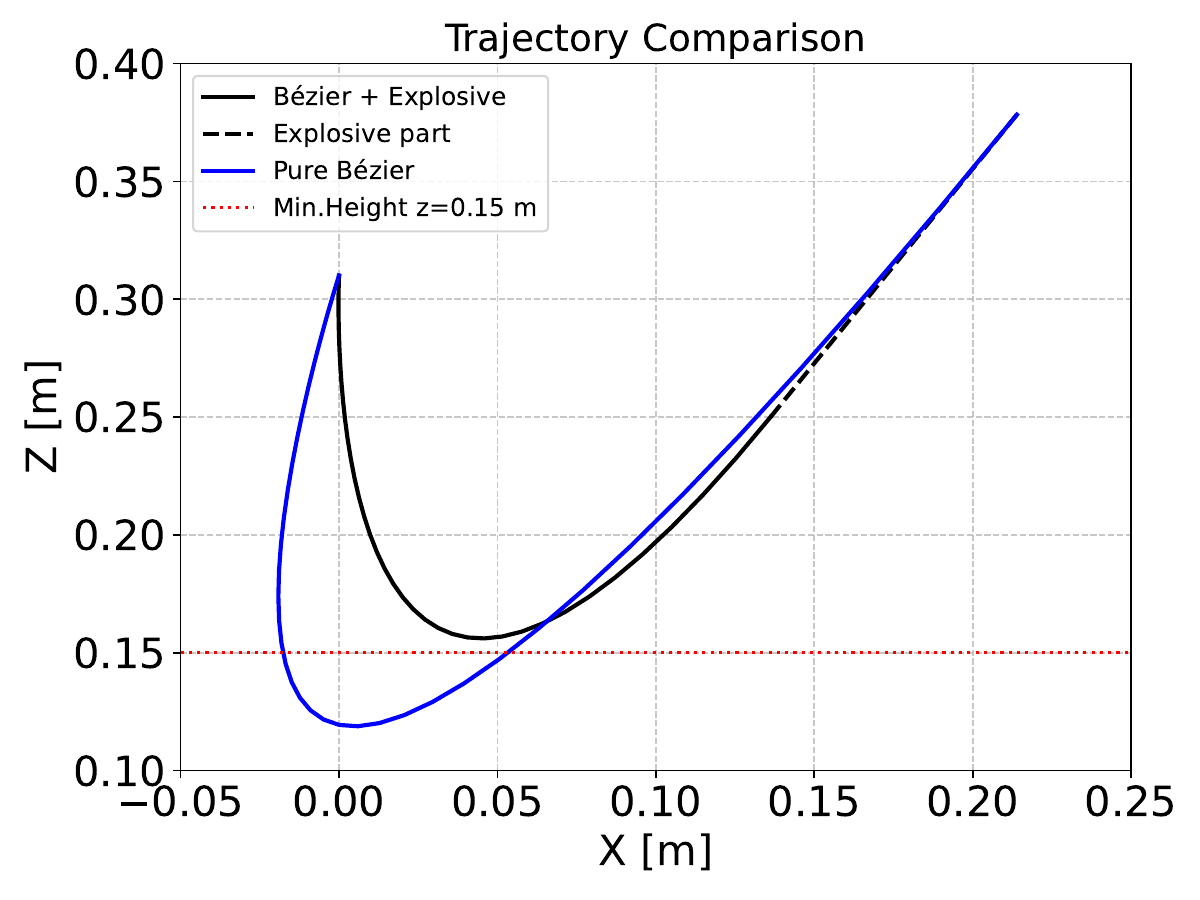}
    \caption{Comparison of the two thrust trajectories computed using only Bézier (blue)  and Bézier+Explosive thrust (black).}
    \label{fig:traj-expl-comp}
\end{figure}
The properties of the Bézier curve can be leveraged to define a reduced set of parameters for the jumping trajectory leveraging the \gls{uarm} assumption. For example, the explosive lift-off velocity ($\mathbf{\dot{c}}_{lo_{e}}$)  can be represented as a scaled version of the Bézier lift-off velocity vector  ($\mathbf{\dot{c}}_{lo_{b}}$). Similarly, the total displacement of the explosive trajectory ($\mathbf{c}_{lo_{e}}$) can be expressed as an offset from the Bézier lift-off position ($\mathbf{c}_{lo_{b}}$) in the direction of the Bézier lift-off velocity.
We introduce the velocity multiplier term \(k\), so the explosive lift-off velocity is defined as:
\begin{equation}
    \mathbf{\dot{c}}_{lo_{e}} = k \mathbf{\dot{c}}_{lo_{b}} \quad k \geq 1
\end{equation}
To compute $\mathbf{c}_{lo_{e}}$  we first evaluate the unit vector based
on the Bézier lift-off velocity to establish the direction of motion:
\begin{equation}
    \mathbf{\hat{\dot{c}}}_{lo_{b}} = \frac{\mathbf{\dot{c}}_{lo_{b}}}{\| \mathbf{\dot{c}}_{lo_{b}} \|}
\end{equation}
Next, we introduce the displacement value \(d\), allowing us to define the explosive lift-off position as:
\begin{equation}
    \mathbf{c}_{lo_{e}} = \mathbf{c}_{lo_{b}} + d \mathbf{\hat{\dot{c}}}_{lo_{b}} \quad d \geq 0
\end{equation}

By using this approach, only two parameters, $k$ and $d$, are introduced instead of six, 
providing sufficient expressiveness without significantly increasing the problem's dimensionality. 
This formulation inherently ensures also that the explosive part is a direct continuation 
of the Bézier trajectory. The policy can still  choose to exclude the contribution of the \gls{uarm} phase for certain jumps by simply setting \(d = 0\). Thus, the \gls{uarm} trajectory can be defined as:

{\small\begin{equation}
    \text{uarm}(t)=
    \left\{    
    \begin{aligned} 
        \mathbf{c}(t)&=\text{lerp}\left(\mathbf{c}_{lo_{b}}, \mathbf{c}_{lo_{e}}, \frac{t - T_{th_{b}}}{T_{th_{e}}} \right)\\
        \mathbf{\dot{c}}(t)&=\text{lerp}\left(\mathbf{\dot{c}}_{lo_{b}}, \mathbf{\dot{c}}_{lo_{e}}, \frac{t - T_{th_{b}}}{T_{th_{e}}} \right)
    \end{aligned}
    \right. T_{th_{b}}\leq t \leq T_{th}  
    \label{eq:uarm_parametrization} \raisetag{2\normalbaselineskip} 
\end{equation}}
where lerp is a linear interpolating function, while, the Bézier position trajectory is expressed as:

{\small \begin{equation}
    \text{Position-Bézier}(t) =
    \left\{
    \begin{aligned}        
        \mathbf{c}(t) &=  \mathbf{B}_p\left(\frac{t}{T_{th_{b}}}, \mathbf{P} \right)\\
        \mathbf{\dot{c}}(t) &= \mathbf{\dot{B}}_p\left(\frac{t}{T_{th_{b}}}, \mathbf{P} \right)
    \end{aligned}
    \right. 0\leq t \leq T_{th_{b}}  
    \label{eq:bezier_parametrization_position}\raisetag{2\normalbaselineskip} 
\end{equation}}
where the control points \eqref{eq:bezier-control-points}, using the new notation, are defined as:

\begin{table}[htbp]
	\centering
	\caption{Position-Bézier Control points}
	\label{tab:bezier-control-points-position}
	\begin{tabular}{l|l}   
		\hline\hline
		Position & Velocity \\
		\hline      
		$\mathbf{P}_{0} = \mathbf{c}_{0}$  & $\mathbf{P^{\prime}}_{0} = \mathbf{\dot{c}}_{0}$\\[2mm]
		$\mathbf{P}_{1} = \frac{T_{th_{b}}}{3}\mathbf{\dot{c}}_{0} + \mathbf{c}_{0}$ & $\mathbf{P^{\prime}}_{1} = \frac{3}{T_{th_{b}}}(\mathbf{P}_{2}-\mathbf{P}_{1})$\\[2mm]
		$\mathbf{P}_{2} = -\frac{T_{th_{b}}}{3}\mathbf{\dot{c}}_{lo_{b}} + \mathbf{c}_{lo_{b}}$ & $\mathbf{P^{\prime}}_{2} = \mathbf{\dot{c}}_{lo_{b}}$\\[2mm]
		$\mathbf{P}_{3} =  \mathbf{c}_{lo_{b}}$ & ~ \\  
		\hline\hline
	\end{tabular}   
\end{table}

Fig. \ref{fig:traj-expl-comp} shows the two resulting trajectories: 
in black the new parameterization method, and  in blue the previous pure Bézier one. In this example, a high final lift-off velocity with a magnitude of \(3 \, \frac{m}{s}\) is required in both cases, with the same thrust time \(T_{th}\). As illustrated, the previous parameterization produces an unfeasible trajectory that violates the minimum height constraint of \(z \geq 0.15 \, \text{m}\). This demonstrates that the new parameterization can generate explosive thrust trajectories while satisfying the system's feasibility constraints.
For the orientation, instead, we employ the Bézier curve parametrization throughout the whole thrust duration $T_{th}$:

{\small\begin{equation}
    \text{Orientation-Bézier}(t) =
    \left\{
    \begin{aligned}        
        \boldsymbol{\Phi}(t) &=  \mathbf{B}_o\left(\frac{t}{T_{th}}, \mathbf{Q} \right)\\
        \boldsymbol{\dot{\Phi}}(t) &= \mathbf{\dot{B}}_o\left(\frac{t}{T_{th}}, \mathbf{Q} \right)
    \end{aligned}
    \right. 0 \leq t \leq T_{th}  
    \label{eq:bezier_parametrization_orientation}
\end{equation}}
with control points defined as:
\begin{table}[htbp]
    \centering  \caption{Orientation-Bézier Control points}
    \begin{tabular}{l|l}   
        \hline\hline
         Position & Velocity \\ 
         \hline      
         & \\
         $\mathbf{Q}_{0} = \boldsymbol{\Phi}_{0}$            & $\mathbf{Q^{\prime}}_{0} = \boldsymbol{\dot{\Phi}}_{0}$\\
         $\mathbf{Q}_{1} = \frac{T_{th }}{3}\boldsymbol{\dot{\Phi}}_{0} + \boldsymbol{\Phi}_{0}$ &       $\mathbf{Q^{\prime}}_{1} = \frac{3}{T_{th}}(\mathbf{Q}_{2}-\mathbf{Q}_{1})$\\
        $\mathbf{Q}_{2} = -\frac{T_{th}}{3}\boldsymbol{\dot{\Phi}}_{lo} + \boldsymbol{\Phi}_{lo}$   & $\mathbf{Q^{\prime}}_{2} = \boldsymbol{\dot{\Phi}}_{lo}$\\
         $\mathbf{Q}_{3} =  \boldsymbol{\Phi}_{lo}$ & \\
         & \\
         \hline\hline
    \end{tabular}   
    \label{tab:bezier-control-points-orientation}
\end{table}
\section{Learning Framework}
\label{sec:rl}
The design of state and action spaces is crucial to ensure that a learned policy can be applied across robots with different morphologies. To achieve this, the state space should include only task-specific terms, such as the robot’s centre of mass (COM) position and orientation, as well as the target position and orientation, while avoiding any joint-space information that would differ between robots. Mapping task-space information to joint space can be effectively handled by inverse kinematics (IK). This approach is complemented by the use of a low-level controller to ensure robust and accurate tracking performance, thereby enabling reusable policies across different robots \cite{shafiee2024manyquadrupeds}.

\subsection{The State }
The state \(s_{t}\) is a representation of the environment at a specific time \(t\) and is fundamental for the policy to determine the appropriate action to perform. 
%
We define the state as the combination of the robot's current configuration (position + pose) and the desired configuration 
that we express  as a displacement relative to the current (starting) position. 
In our setup, for simplicity, the robot always starts from a fixed joint configuration, meaning its initial position is constant. As a result, including this initial configuration in the state representation is redundant. Therefore, the state space can be more efficiently reformulated as:
\begin{equation}
    \mathbf{s} = (\Delta\mathbf{c}, \Delta\mathbf{\Phi}) \in \mathbb{R}^{6}
\end{equation}
Where  $\Delta\mathbf{c} = \mathbf{c}_{tg} - \mathbf{c_0}$, $\Delta\boldsymbol{\Phi} = \mathbf{\Phi}_{tg} - \mathbf{\Phi_0}$. This new state space directly captures the displacement command in both position and orientation, streamlining the policy's decision-making process.
%
\subsection{The Action}
The parameters of the thrust trajectory represent the action, i.e., the values that the policy must predict 
based on the input state. 
Considering the proposed thrust trajectory parameterization for position \eqref{eq:uarm_parametrization}, \eqref{eq:bezier_parametrization_position}, and orientation \eqref{eq:bezier_parametrization_orientation} 
the following parameters are required to compute the trajectory:
\begin{itemize}
    \item \(T_{th_{b}}\): the time interval in the thrust for which the initial Position-Bézier is used
    \item \(\mathbf{c}_{lo_{b}}\): the lift-off position used for the Position-Bézier trajectory
    \item \(\mathbf{\dot{c}}_{lo_{b}}\): the lift-off velocity used for the Position-Bézier trajectory
    \item \(\mathbf{c}_{lo_{e}}\): the lift-off position used for the \gls{uarm} trajectory, derived using the displacement term $d$
    \item \(\mathbf{\dot{c}}_{lo_{e}}\): the lift-off velocity used for the \gls{uarm} trajectory, derived using the velocity multiplier term $k$
    \item \(\mathbf{\Phi}_{lo}\): the lift-off orientation used for the Orientation-Bézier trajectory
    \item \(\mathbf{\dot{\Phi}}_{lo}\): the lift-off Euler rates  used for the Orientation-Bézier trajectory
\end{itemize}
Then, $T_{th_{e}}$ is computed analytically via \eqref{eq:uarm-system-new} and the total thrust time $T_{th}$ is obtained by \eqref{eq:tot-time}.
The dimensionality of the action space is thigly related to the size of the exploration space and thus on learning performance. A smaller action space reduces the exploration area, simplifying the complexity of the mapping and speeding up the learning process. A trick to  reduce the size of the action space is to express both the Bézier \gls{com} lift-off position \(\mathbf{c}_{lo_{b}}\) and linear velocity \(\mathbf{\dot{c}}_{lo_{b}}\) in \textit{spherical coordinates}. 

Due to the natural decoupling between linear and angular dynamics we treat the linear and angular motions separately. 
If no angular motion is present,  due to the ballistic nature of the flight phase,  the entire jump trajectory lies within the plane connecting the initial and target locations with  the yaw angle $\varphi_{jp}$ (which defines the orientation of the jumping plane in the $X-Y$ frame) remaining constant throughout the jump. 
This angle can be computed analytically from the initial and target positions restricting the coordinates of 
the lift-off configuration  to a convex, two-dimensional space:
\begin{equation}
    \left\{\begin{matrix}
    \begin{aligned}
     \mathbf{c}_{lo_{b},x} &= r_p \cos(\theta_p) \cos(\bar{\varphi}) \\
     \mathbf{c}_{lo_{b},y} &= r_p \cos(\theta_p) \sin(\bar{\varphi}) \\
     \mathbf{c}_{lo_{b},z} &= r_p \sin(\theta_p)
    \end{aligned}
    \end{matrix}\right.
    \quad
    \left\{\begin{matrix}
    \begin{aligned}
     \mathbf{\dot{c}}_{lo_{b},x} &= r_{v} \cos(\theta_{v}) \cos(\bar{\varphi}) \\
     \mathbf{\dot{c}}_{lo_{b},y} &= r_{v} \cos(\theta_{v}) \sin(\bar{\varphi}) \\
     \mathbf{\dot{c}}_{lo_{b},z} &= r_{v} \sin(\theta_{v})
    \end{aligned}
    \end{matrix}\right.
    \label{eq:spherical-action}
\end{equation}
As shown in Fig.\ref{fig:action-simp}, the lift-off position $\mathbf{c}_{lo_{b}}$ is defined by the extension radius $r_p$, pitch angle $\theta_p$, and pre-computed yaw angle $\bar{\varphi}$. The lift-off velocity $\mathbf{\dot{c}}_{lo_{b}}$ shares the same yaw  $\bar{\varphi}$ and is described by its magnitude $r_{v}$ and pitch angle $\theta_{v}$.
\begin{figure}[htbp]
      \includegraphics[width=0.38\columnwidth]{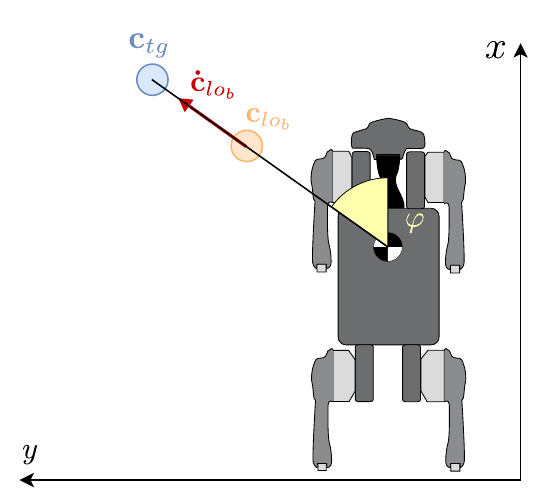}
      \includegraphics[width=0.6\columnwidth]{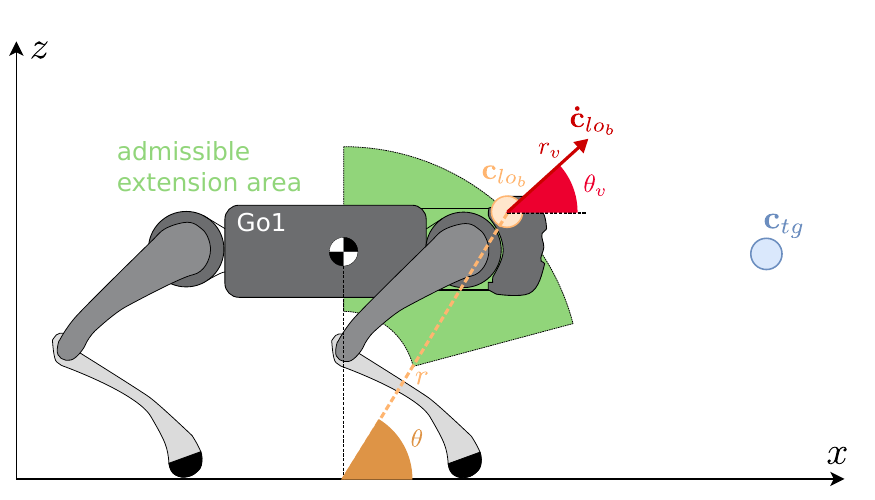}
      \caption{Top and side views of the action in spherical coordinates for the quadruped robot}
      \label{fig:action-simp}
\end{figure}

Given this, we can now define the  action space, which encompasses both 
the position and angular components of the thrust trajectory, as follows:
\begin{equation}
    \begin{aligned}
        \mathbf{a} & = (T_{th_{b}}, r_p , \theta_p, r_{v}, \theta_{v},  k, d,\mathbf{\Phi}_{lo,\psi}, \mathbf{\Phi}_{lo,\theta}, \\&~~~~~~~ \mathbf{\Phi}_{lo,\varphi}, 
        \mathbf{\dot{\Phi}}_{lo,\psi}, \mathbf{\dot{\Phi}}_{lo,\theta}, \mathbf{\dot{\Phi}}_{lo,\varphi}
        ) \in \mathbb{R}^{13}
    \end{aligned}
\end{equation}

The ranges of these values can be further constrained using domain knowledge to further narrow the search space.  
The thrust duration ($T_{th}$) must be short but sufficient for precise, controlled lift-off; 
The radius $r_p$ has to be smaller than a value $r_{max}$   to prevent boundary singularity due to leg
over-extension, and greater than a value $r_{min}$  to avoid complete leg retraction. 
The bounds on the velocity $\dot{c}_{lo}$, represented by $r_{v}$  and $\theta_{v}$ are set to rule out jumps that involve excessive foot slippage and useless force effort. 
Specifically, restricting $\theta_{v,min}$ to be positive ensures a non-negligible vertical component for the velocity, while bounding $\theta_v$ to the positive quadrant secures that the lift-off velocity will be oriented "toward" the target. The limits on lift-off roll/pitch/yaw angles ($\boldsymbol{\Phi}_{lo}$) limit the amount of possible  re-orientation adjustments during the thrust to avoid losing stability and maneuverability, especially on uneven terrain; while the Euler rates ($\boldsymbol{\dot{\Phi}}_{lo}$) at lift-off determine  the reorientation of the robot during the subsequent flight phase; the Velocity Multiplier ($k$) amplifies thrust explosiveness, but must stay within trackable limits; the final displacement ($d$) adds travel distance and affects thrust phase duration. Given their physical meaning, we can impose specific ranges on these values (see Table \ref{tab:acrion_params_ranges}).  
Note that the above reductions, done at the level of action design, prevent the agent from exploring trajectories that are physically impossible, reducing the search space without any loss in terms of optimality.

\subsection{A Physically Informative Reward Function}
The reward function serves as the sole indicator that evaluates the effectiveness of an action in achieving the task goal. 
This comprises of a \textit{physical} component which is a summation of penalties which is  subtracted from a  \textit{target} 
component (positive) that assesses the achievement of the task's main objective:  i.e. how closely the robot landed to the target location. 

The core idea behind our reward design is that achieving the maximum reward related to the task goal requires adherence to and minimization of a set of physical system constraints. We have previously introduced the concept of path constraints, but to recap: a path constraint is a restriction that must be satisfied throughout the entire thrust phase. Any violation of these constraints results in a penalty. These path constraints are evaluated at each time step, and violations are penalized using the linear activation function \(A(x,\underline{x},\overline{x})\) defined in \eqref{eq:activation-func}. 
\begin{equation}
    A(x,\underline{x},\overline{x}) =
    \begin{vmatrix}
    \min(x-\underline{x},0) + \max(x-\overline{x},0)
    \end{vmatrix}
    \label{eq:activation-func}
\end{equation}
All penalties are accumulated into a feasibility cost $C_{f}$.
The set of path constraints, essential for ensuring the physical feasibility and safe execution of the jump, are evaluated at every time step and include:  joint position, velocity, and torque limits, friction cone, unilaterality, and staying away from singularity. By minimizing the penalties associated with any violations,  the policy ensures that the robot operates in a physically feasible and reliable manner throughout the entire jump.
However, path constraints are not sufficient to fully guide the policy towards optimal jump behavior. Additional costs are needed to promote desirable behaviors while penalizing unwanted actions. These costs differ from path constraints acting like behavioral “nudges” as they focus more on guiding specific behavior during the jump.

\begin{itemize}
\item \textbf{Lift-off tracking error} ($C_{lo}$): is the deviation between the reference and the actual twist at lift-off. This cost promotes the fact that the trajectory performed by the robot is accurately tracked by the low-level controller.
\item \textbf{Target orientation error} ($C_{\mathbf{\Phi}_{tg}}$): Instead of incorporating the orientation error in the positive reward, we treat it as a penalty. While the primary goal is to reach the target location, minimizing the orientation error at landing is crucial for task success because it enables to avoid catastrophic falls after touchdown.
\item \textbf{Touchdown bounce penalization} ($C_{\Delta x}$): This regularization cost penalizes multiple bounces after landing, which can occur because we do not employ a sophisticated landing controller (e.g., as in \cite{roscia2023reactive}) and instead rely on simple joint impedance control. Excessive horizontal velocity at touchdown may lead to re-bouncing behavior before the robot comes to a complete stop, or even destabilize it, potentially causing a fall. To mitigate this, the cost is set to a high default value when no valid touchdown is detected - such as when the robot makes contact with the ground using body parts other than the feet. Otherwise, the cost is computed proportionally to the distance between the first detected touchdown and the robot's position at the end of the episode (timeout).
\item \textbf{Touchdown angular velocity penalization} ($C_{\boldsymbol{\dot{\Phi}}_{td}}$): This penalty discourages high angular velocities at touchdown, which could cause the robot to lose balance or suffer mechanical damage, leading to a failed jump.
\item \textbf{Action limit penalization} ($C_{ppo}$): 
Differently from other \gls{rl} algorithms that bound the action range to $[-1,1]$, thanks to a   $\tanh$ layer, 
the \gls{ppo} algorithm computes actions by sampling from a normal distribution \(a \sim \mathcal{N}(\mu,\sigma^2)\), which is not inherently constrained. To prevent actions from exceeding the defined range, clipping is introduced. Although this operation does not compromise the mathematical formulation of \gls{ppo}, the choice of clipping range significantly impacts policy performance \cite{fujita2018clipped}. Therefore, determining the appropriate clipping range for \gls{ppo} becomes an important hyperparameter (see Table \ref{tab:hyper_params}). 
The $C_{ppo}$ penalty informs the policy when an action exceeds the predefined bounds, 
thereby helping the policy to stay within the safe and feasible range.
\end{itemize}

These additional costs guide the policy towards safer, more stable, and physically feasible jump execution, 
further refining the policy's behavior during training.
Regarding the positive reward task, we refer to it as the \textit{landing target reward } function, denoted by $R_{lt}$. 
This term measures how close the final landing location (evaluated at the episode's timeout) is to the target location.
In this term, we also want to weigh errors differently depending on the jumping magnitude. The goal is to better distinguish the impact of landing errors $\text{e}_{tg}$ based on the jump distance $\Delta{\mathbf{c}}$. For example, an error of 0.2 meters is significant for a 0.1 meter-long jump but less severe for a 1 meter jump. Therefore, for a given distance error, the reward should increase proportionally with the desired jump distance. 
The revised landing target reward function is defined as follows:

%
%
%

{\small
\begin{equation}
    R_{lt} = 1_{\mathbb{R}^{+}}\left[ \exp\left(-\frac{\text{e}_{tg}}{ \sigma_e}\right) \exp\left(\frac{\Delta{\mathbf{c}}}{\sigma_d}\right) \right]
\end{equation}}

with 
\begin{equation*}
    \text{e}_{tg} = \| \mathbf{c}_{tg}-\mathbf{c}  \| \quad \Delta\mathbf{c} = \|\mathbf{c}_{tg} - \mathbf{c}_{0}\|
\end{equation*}
where $\sigma_e$ is a hyperparameter scaling the influence of the landing error, and $\sigma_d$  the contribution of the jump distance.
The first negative exponential achieves its maximum for zero landing error and rewards \textit{accuracy} while the second (positive exponential) amplifies the previous factor according to the jump difficulty (i.e. to the same error is associated a higher reward if the jump distance is higher). This formulation makes the reward sensitive to both the landing error and the magnitude of the jump, ensuring that the policy is encouraged to perform well across different jump distances.
In Fig. \ref{fig:targ-rew}, the surface plot of this reward function is depicted.
As can be observed, the reward increases as the landing error approaches zero. Furthermore, for the same landing error, the reward value increases proportionally with the desired jump distance, reflecting the relative importance of accuracy across different distances.
\begin{figure}[b]
\centering
    \includegraphics[width=0.8\columnwidth]{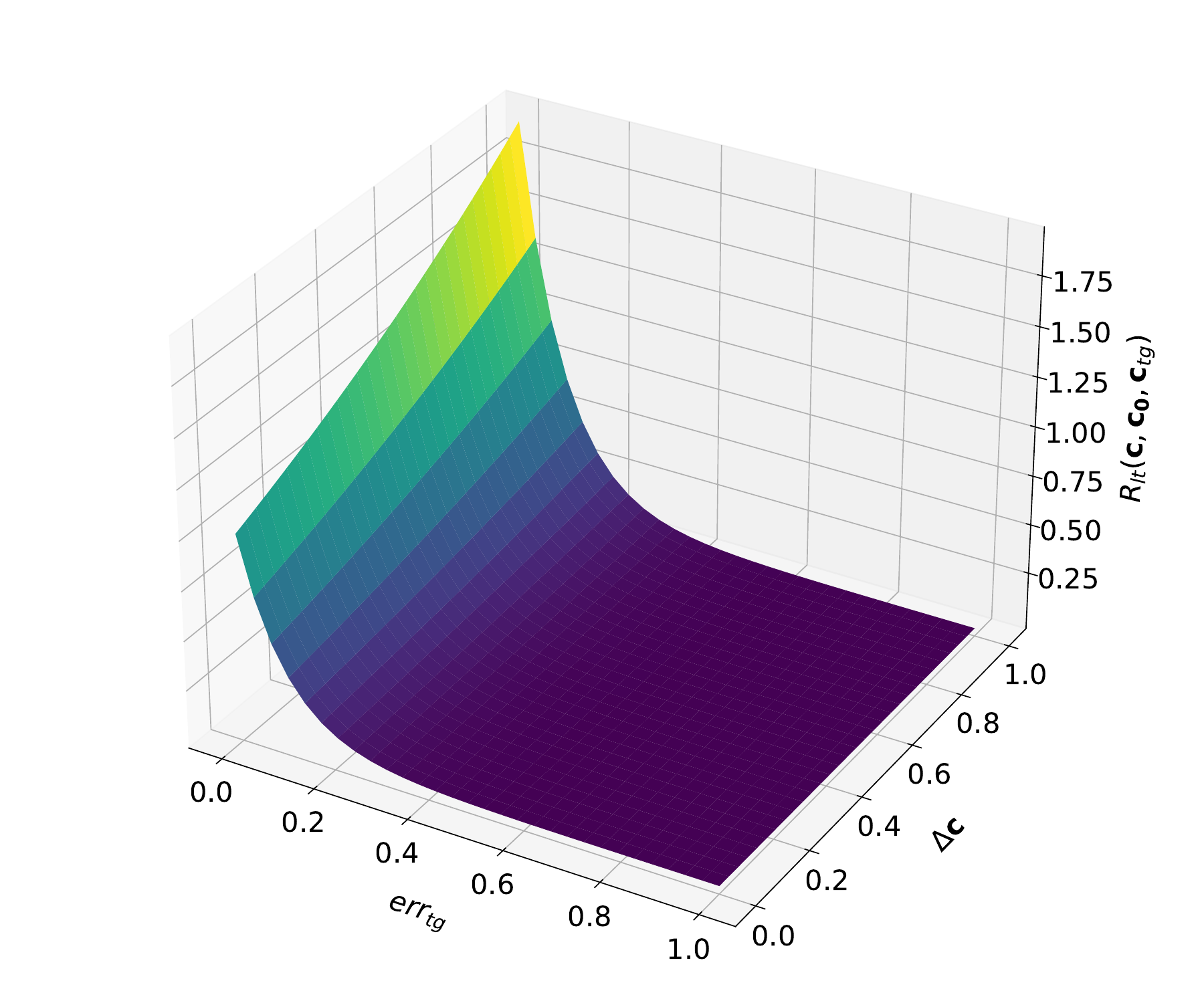}
    \caption{Landing target reward function surface function of landing error and jump distance.}
    \label{fig:targ-rew}
\end{figure}
\noindent
Finally, the total  reward function is \cite{atanassov2024curriculumbased}:
\begin{equation}
R = R_{lt}(\mathbf{c},\mathbf{c_{0}}, \mathbf{c}_{tg}) e^{-\left(\sum C_{f,i}\right)^2}
\end{equation}
In this formulation, the reward remains positive and is dynamically scaled based on the penalty costs. 
As the sum of penalties increases, the negative exponent drives the reward toward zero, effectively 
suppressing any positive contribution. Conversely, when penalties are minimal, the exponential term approaches 
one, preserving the full value of the landing target reward. 
This design keeps the reward function informative across the learning process, fostering better convergence and policy performance.
To maximize the reward, the policy must minimize the cost penalties, thereby ensuring that 
all system constraints are respected. This guarantees that the proposed trajectory 
is both feasible and trackable, enabling the robot to perform optimal jumps within its physical
limitations, ensuring safe and efficient deployment of the learned behaviors in real-world scenarios.
%
\section{Results }
\label{sec:results}
%
%
\subsection{Simulation Setup}
In this section, we present the results of our work, validated through simulation-based evaluation of the learned policy.
We report the hardware specifications, as they critically affect training time and parallelization performance.
Training and evaluation were performed on a desktop with an Nvidia RTX 3060 GPU with 12 GB of VRAM. 
The policy was trained using a modified version of the Orbit framework \cite{mittal2023orbit} compatible with our RL loop, leveraging parallel simulation of 4096 robots.
The training has been performed both for the Go1 and the Aliengo, but, henceforth, we will detail the results solely for the Go1 robot.
The robot controller and simulator parameters are detailed in Table \ref{tab:robot_params}, while the 
hyperparameters used in the training are listed in Table \ref{tab:hyper_params}.
Table \ref{tab:acrion_params_ranges} specifies the ranges of action parameters defined 
from physical intuition with the aim to further restrict the search space. 
\ifdefined\springer
	\begin{table}[htb!]
	\centering
	\caption{Go1 Robot and Simulator parameters}
	\begin{tabular}{c| p{4.5cm} | c } 
	    \hline
	    \textbf{Variable} & \textbf{Name} & \textbf{Range}\\
	    \hline    \hline
	     $m$ &  Robot Mass [$kg$] & 13\\
	    \hline
	     P & Proportional gain & 50 \\
	    \hline
	     D & Derivative gain & 0.8 \\
	    \hline
	     $\vect{q}_0$ & Nominal configuration  [$rad$] & \mat{0&-0.75&1.5} \\
	    \hline
	     $dT$ & Simulator time step [$s$] & 0.001\\
	    \hline
	     $\tau_{max}$ & Max Joint torque [$Nm$] & \mat{23.7&23.7&35.5} \\
	     \hline
	          $\dot{q}_{max}$ & Max Joint vel. [$m/s$] &\mat{20&20&30} \\
	    \hline
	    \hline
	\end{tabular}
	\label{tab:robot_params}
	\end{table}
\else
	\begin{table}[htb!]
			\centering
			\caption{Go1 Robot and Simulator parameters}
			\resizebox{\columnwidth}{!} {
				\begin{tabular}{c| p{4.5cm} | c } 
					\hline
					\textbf{Variable} & \textbf{Name} & \textbf{Range}\\
					\hline    \hline
					$m$ &  Robot Mass [$kg$] & 13\\
					\hline
					P & Proportional gain & 50 \\
					\hline
					D & Derivative gain & 0.8 \\
					\hline
					$\vect{q}_0$ & Nominal configuration  [$rad$] & \mat{0&-0.75&1.5} \\
					\hline
					$dT$ & Simulator time step [$s$] & 0.001\\
					\hline
					$\tau_{max}$ & Max Joint torque [$Nm$] & \mat{23.7&23.7&35.5} \\
					\hline
					$\dot{q}_{max}$ & Max Joint vel. [$m/s$] &\mat{20&20&30} \\
					\hline
					\hline
			\end{tabular}}
			\label{tab:robot_params}
	\end{table}
\fi
\ifdefined\springer
	\begin{table}[tbp]
		\centering
			\caption{\gls{ppo} hyper-parameters}
				\begin{tabular}{p{6cm} | c } 
					\hline
					\textbf{Name} & \textbf{Value}\\
					\hline    \hline 
					Episode timout             & 1.5s  \\    
					\hline 
					Initial noise std             & 1.0  \\    
					\hline
					Network dimension             & [512,256,128] \\    
					\hline
					Activation layer             & ELU \\    
					\hline
					PPO Clip             & 0.2 \\    
					\hline
					Entropy coefficient             & 0.01 \\    
					\hline
					Number of learning epochs           & 10 \\    
					\hline
					Learning rate             & 1e-3 \\    
					\hline
					Gamma            & 0.99 \\    
					\hline
					Lam            & 0.95 \\    
					\hline
					Desired KL            & 0.01 \\    
					\hline
			\end{tabular}
			\label{tab:hyper_params}
\end{table}
\else
{\small
	\begin{table}[tbp]
	\centering
	\caption{\gls{ppo} hyper-parameters}
	\resizebox{\columnwidth}{!} {
	\begin{tabular}{p{6cm} | c } 
	        \hline
	    \textbf{Name} & \textbf{Value}\\
	    \hline    \hline 
	    Episode timout             & 1.5s  \\    
	    \hline 
	    Initial noise std             & 1.0  \\    
	    \hline
	    Network dimension             & [512,256,128] \\    
	    \hline
	    Activation layer             & ELU \\    
	    \hline
	    PPO Clip             & 0.2 \\    
	    \hline
	    Entropy coefficient             & 0.01 \\    
	    \hline
	    Number of learning epochs           & 10 \\    
	    \hline
	    Learning rate             & 1e-3 \\    
	    \hline
	    Gamma            & 0.99 \\    
	    \hline
	    Lam            & 0.95 \\    
	    \hline
	    Desired KL            & 0.01 \\    
	    \hline
	\end{tabular}}
	\label{tab:hyper_params}
	\end{table}}
\fi
%
%
%
The learning episode starts from the same initial default configuration $\vect{q}_0$ continues until a timeout, 
requiring the robot to regain balance and return to a nominal state after landing. 
To maintain a unified  control strategy, we decided not to include the \gls{lc} as
part of the simulation in the training phase, relying on a simplified impedance strategy.
During the thrust phase, a default joint stiffness is set in the PD controller plus a \gls{wbc} for gravity compensation. 
Since all four feet are in contact with the ground throughout the thrust phase, the control strategy tracks the desired Cartesian trajectory and compensates for gravity. Contact is maintained by ensuring the fulfillment of both the friction cone and unilateral constraints during training.
After thrust, the stiffness is reduced, the \gls{wbc} is disabled, and the legs are retracted until the jump apex. After that, the legs extend back to the default configuration. Finally, upon touchdown, \gls{wbc} is reactivated to compensate for gravity.
To prevent torque discontinuities during configuration changes in the airborne phase, 
Cubic Hermite Spline Interpolation is used to smoothly transition joint references between extended and retracted configurations.
In a future extension of this work, we plan to also introduce a separate \gls{nn} trained using an \gls{e2e}
approach, specifically tasked to adjust the joint configuration at each time step emulating a \gls{lc}. 
For each completed episode, one \gls{ppo} training step is performed, amounting to a total of 2000 training steps. 
Training the policy under this setup took approximately 15 hours. However, it's important to emphasize 
the sample efficiency of our approach when compared to typical \gls{e2e} methods, 
such as those available in the Orbit framework. These methods generally require the recorded data of 20 to 25 simulation steps to complete a single training step. In contrast, our method 
performs a training step with just one env step, corresponding to the entirety of the 
episode since the episode is composed of only one action. This means that, given an equivalent 
number of policy training steps, our approach uses 20 times less data, 
highlighting the remarkable sample efficiency of the proposed framework, despite the training 
time is comparable (the overhead due to the simulation of the robots will be the same). 

\ifdefined\springer
	\begin{table}[tb!]
	\centering
	\caption{Ranges of action parameters}
	\begin{tabular}{c| c | c } 
	    \hline
	    \textbf{Variable} & \textbf{Name} & \textbf{Range}\\
	    \hline    \hline
	     $T_{th_{b}}$&  Time of Bezier thrust [s]& [0.4, 1.0]\\
	    \hline
	     $r_p$ & Extension radius of Bezier thrust[m] & [0.2, 0.4] \\
	    \hline
	     $\theta_p$ & Position Pitch angle of Bezier thrust[rad] & [$\pi$/4, $\pi$/2]\\
	    \hline
	     $r_v$ & \gls{com} Velocity magnitude of Bezier thrust [m/s] & [0.5, 5]\\
	    \hline
	     $\theta_p$ & \gls{com} Velocity Pitch angle of Bezier thrust [rad] & [-$\pi$/6, $\pi$/2] \\
	    \hline
	     $k$ & Velocity multiplier for explosive thrust& [1,3] \\
	    \hline
	    $d$ & Position displacement for explosive thrust& [0,0.3] \\
	    \hline
	    $\mathbf{\Phi}_{lo,\psi}$ & The pitch angle at lift-off & [-$\pi$/6, $\pi$/6] \\
	    \hline
	    $\mathbf{\Phi}_{lo,\theta}$ & The roll angle at lift-off & [-$\pi$/6, $\pi$/6] \\
	    \hline
	    $\mathbf{\Phi}_{lo,\varphi}$ & The yaw angle at lift-off  & [-$\pi$/4, $\pi$/4] \\
	    \hline
	    $\mathbf{\dot{\Phi}}_{lo,\psi}$ &  The angular velocity around the roll axis & [-1,1] \\
	    \hline
	    $\mathbf{\dot{\Phi}}_{lo,\theta}$ &  The angular velocity around the pitch axis & [-1,1] \\
	    \hline
	    $\mathbf{\dot{\Phi}}_{lo,\varphi}$ & The angular velocity around the yaw axis & [-4, 4] \\
	    \hline
	\end{tabular}
	\label{tab:acrion_params_ranges}
	\end{table}
\else
	\begin{table}[tb!]
	\centering
	\caption{Ranges of action parameters}
	\resizebox{\columnwidth}{!} {
		\begin{tabular}{c| c | c } 
			\hline
			\textbf{Variable} & \textbf{Name} & \textbf{Range}\\
			\hline    \hline
			$T_{th_{b}}$&  Time of Bezier thrust [s]& [0.4, 1.0]\\
			\hline
			$r_p$ & Extension radius of Bezier thrust[m] & [0.2, 0.4] \\
			\hline
			$\theta_p$ & Position Pitch angle of Bezier thrust[rad] & [$\pi$/4, $\pi$/2]\\
			\hline
			$r_v$ & \gls{com} Velocity magnitude of Bezier thrust [m/s] & [0.5, 5]\\
			\hline
			$\theta_p$ & \gls{com} Velocity Pitch angle of Bezier thrust [rad] & [-$\pi$/6, $\pi$/2] \\
			\hline
			$k$ & Velocity multiplier for explosive thrust& [1,3] \\
			\hline
			$d$ & Position displacement for explosive thrust& [0,0.3] \\
			\hline
			$\mathbf{\Phi}_{lo,\psi}$ & The pitch angle at lift-off & [-$\pi$/6, $\pi$/6] \\
			\hline
			$\mathbf{\Phi}_{lo,\theta}$ & The roll angle at lift-off & [-$\pi$/6, $\pi$/6] \\
			\hline
			$\mathbf{\Phi}_{lo,\varphi}$ & The yaw angle at lift-off  & [-$\pi$/4, $\pi$/4] \\
			\hline
			$\mathbf{\dot{\Phi}}_{lo,\psi}$ &  The angular velocity around the roll axis & [-1,1] \\
			\hline
			$\mathbf{\dot{\Phi}}_{lo,\theta}$ &  The angular velocity around the pitch axis & [-1,1] \\
			\hline
			$\mathbf{\dot{\Phi}}_{lo,\varphi}$ & The angular velocity around the yaw axis & [-4, 4] \\
			\hline
	\end{tabular}}
	\label{tab:acrion_params_ranges}
	\end{table}
\fi
We define the training region for the policy as a cuboid space with the X-axis ranging from -0.6 to 1.2 meters, 
the Y-axis from -0.6 to 0.6 meters, and the Z-axis from -0.4 meters to 0.4 meters,
allowing jumps up to 40 cm elevation. This setup ensures learning omnidirectional jumps to both elevated and depressed targets, 
providing a broad exploration of possible landing configurations. 
For the angular component, the training region is limited to a maximum of 15 degrees for roll and pitch, ensuring the robot can handle minor tilts during the jump. 
However, the yaw angle, which is of greater importance in this work, is allowed to vary significantly, covering a range from -90 to 90 degrees. This range ensures the policy can perform jumps with diverse orientations, preparing the robot for more dynamic scenarios. 

In Fig. \ref{fig:rew-pen}, we present both the total reward curve 
and the cumulative penalty sums for each training step. The total reward (upper plot) converges after only 2000 training steps. 
A crucial observation is the even faster convergence (after 250 episodes) of the cumulative penalties (lower plot) to a small value, suggesting that the learned policy quickly adapts to producing physically feasible actions and adheres to system constraints. This serves as strong evidence that our reward formulation is effective, driving the policy towards solutions that satisfy the constraints while optimizing the task’s objectives.
\begin{figure}[htbp]
     \centering
         \includegraphics[width=0.8\columnwidth]{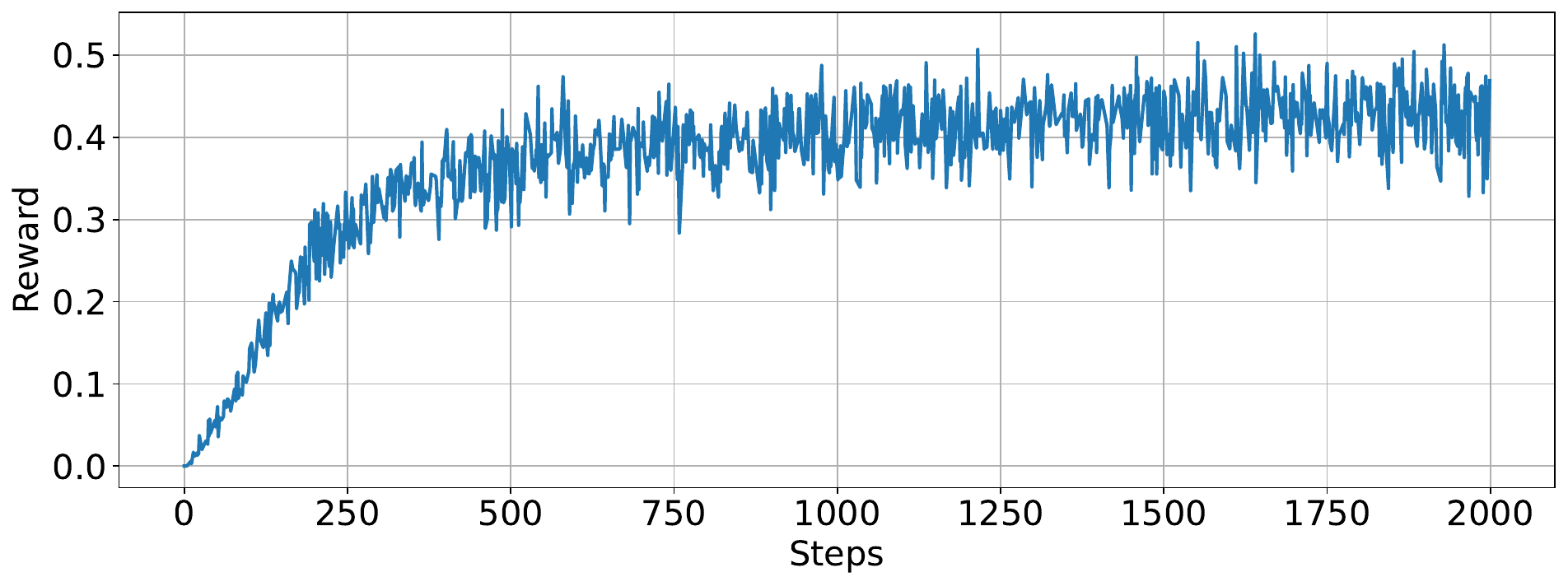}
         \includegraphics[width=0.82\columnwidth]{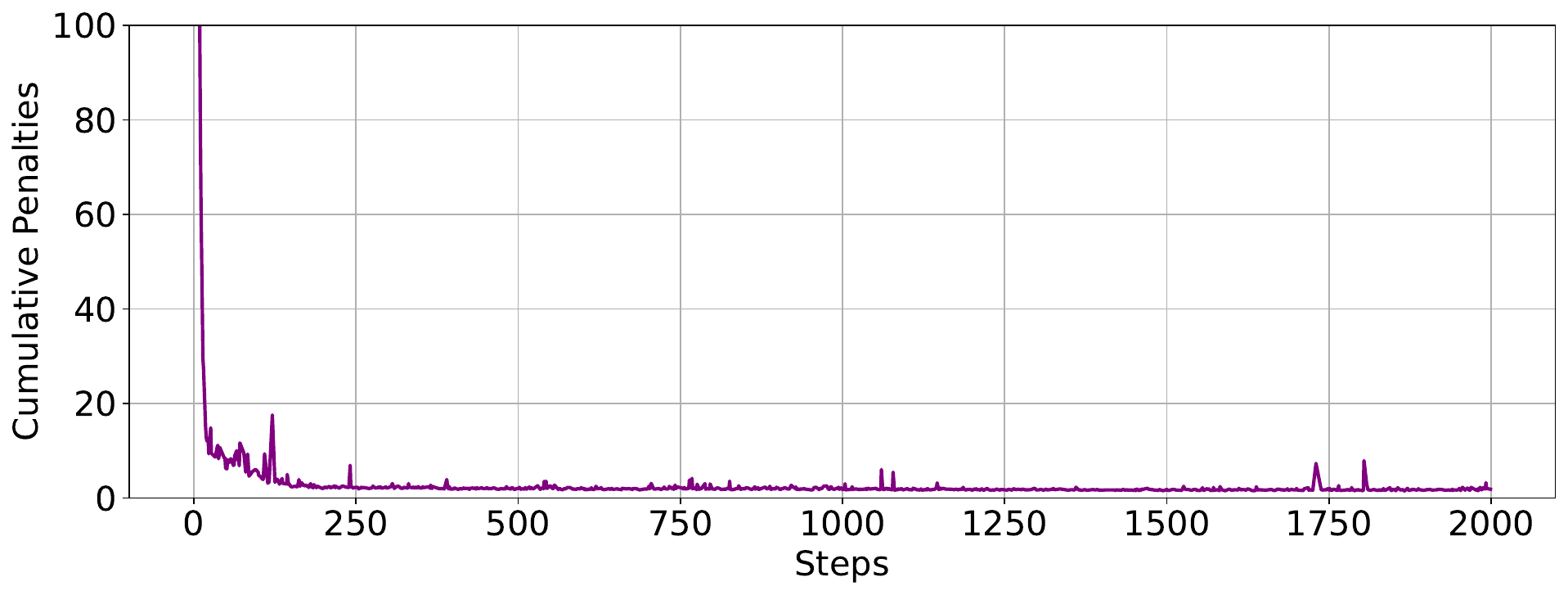}
        \caption{Training curves over training episodes for the Go1 robot:  total reward  (top) and cumulative penalties  (bottom), highlighting the policy's improvement and its adherence to physical constraints.}
        \label{fig:rew-pen}
\end{figure}
\subsection{Physical feasibility check}
\label{sec:physical_check}
The physically informed  nature of our action enables us to perform a safety check on the action $\vect{a}$ 
proposed by the \gls{rl} policy, adding a certain level of explainability with respect to fully data-driven approaches. 
This feasibility check can be employed as an ``a posteriori''   
safety feature in the \textit{inference} phase,  
to check if a predicted action will lead to unsafe results. 
For instance, if the given \gls{com} vertical velocity is not sufficient to reach 
the target height, the action will not produce a physically meaningful jump. 
This can be computed by obtaining the time to reach the apex $T_{f,up} = \vect{\dot{c}}_{{lo},z}/g$ 
and substituting it in the ballistic equation:
\begin{align}
    \bar{\vect{c}}_z(T_{fup}) &= \vect{c}_{{lo},z} + \vect{\dot{c}}_{{lo},z}T_{fup} + \frac{1}{2}(-g)T_{fup}^2  
\label{eq:apex}
\end{align}
This results in $\bar{\vect{c}}_z(T_{fup}) = \vect{c}_{{lo},z} +  \frac{1}{2}\frac{\vect{\dot{c}}_{{lo},z}^2}{g}$, 
which is the \textit{apex} elevation. If $\vect{c}_{tg,z}>\bar{\vect{c}}_z(T_{fup})$,  
the action can be aborted early without performing the jump, and high-level strategies could be 
adopted to relax the jumping requirements (e.g., lower the target height). 
This mechanism enhances safety and reliability by preventing the execution of jumps that could lead to failures or unsafe outcomes.

%
%
\subsection{Validation of Omnidirectional Jumps}
In this section, we evaluate the performance of the learned policy in executing 
omnidirectional jumps on elevated/depressed terrains as well as its
effectiveness in controlling orientation during dynamic jumps. 
The ability to jump in multiple directions, combined with the capability to adapt to varying 
surface heights, enhances the system's versatility for real-world scenarios, 
where simple, single-direction jumps or fixed-height landings are often insufficient.
%
Specifically, this validation will encompass forward, backward, 
and lateral jumps, as well as involving upward or downward motions on/from uneven surfaces.
Rapid in-place/composite jumps involving yaw rotations will also be assessed. 
Through this validation, we aim to show the policy's precision 
in reaching target positions and orientations.
All the simulations have been performed for the Unitree Go1 robot.
\subsubsection{Flat Jumps and Feasible Region}
For this evaluation, we conducted 8192 jumps at various target positions, all with zero height displacement (flat jumps). 
The targets are sampled from the previously established training region, providing a comprehensive evaluation of the learned policy. 
The concept of the feasible region refers to the maximum area in which the robot 
can perform jumps while keeping the landing error below a specified threshold \cite{bussola2024efficient}. 
The absolute landing error (i.e. the distance between the \textit{actual} landing position and the intended target) was calculated for each sample. 
Jumps were classified as acceptable or not based on a landing error below a threshold of 0.2 meters.

As depicted in Fig. \ref{fig:feasible-region}, the computed feasible region illustrates the performance of the robot across the tested positions. 
A key observation is the symmetry of performance with respect to the X-axis, which reflects the symmetrical configuration of the robot’s legs. This symmetry confirms that the robot achieves consistent jumping accuracy for lateral displacements, with no noticeable difference between 
rightward and leftward jumps. Although the jumps in the backward direction were not tested over long distances, 
the results suggest similar performance to frontward jumps.
The logarithmic color scale used in the figure emphasizes the subtle differences in landing precision, 
 highlighting that the robot's best performance is within the 0.5-meter radius from the origin, 
where it achieves exceptional accuracy. Landing errors in this range are consistently below 5 cm, 
showcasing the high precision and of the learned policy for short to medium-range jumps.
\begin{figure}[th!]
\centering
    \includegraphics[width=1.0\columnwidth]{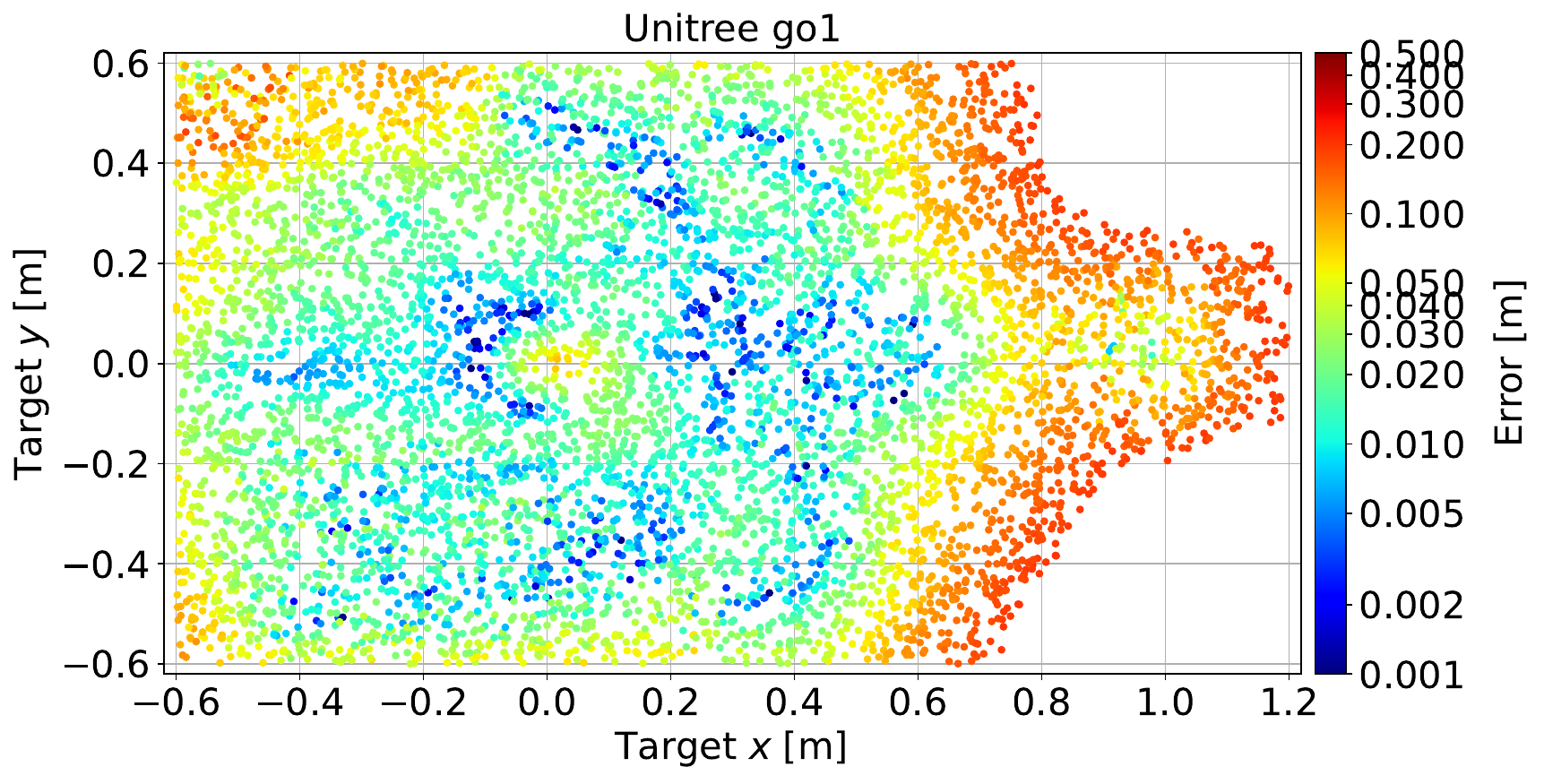}
    \caption{Feasible region for flat omni-directional jumps (without change in orientation) for the Go1 robot: set of targets with landing error below 0.2 m.}
    \label{fig:feasible-region}
\end{figure}
To gain a clearer understanding of the robot's behavior during longer jumps, Fig. \ref{fig:actual-vs-target} 
presents a comparison between the target and actual distances achieved by the robot for both front (left plot) and back jumps (right plot).
\begin{figure}[th!]
\centering
    \includegraphics[width=1.0\columnwidth]{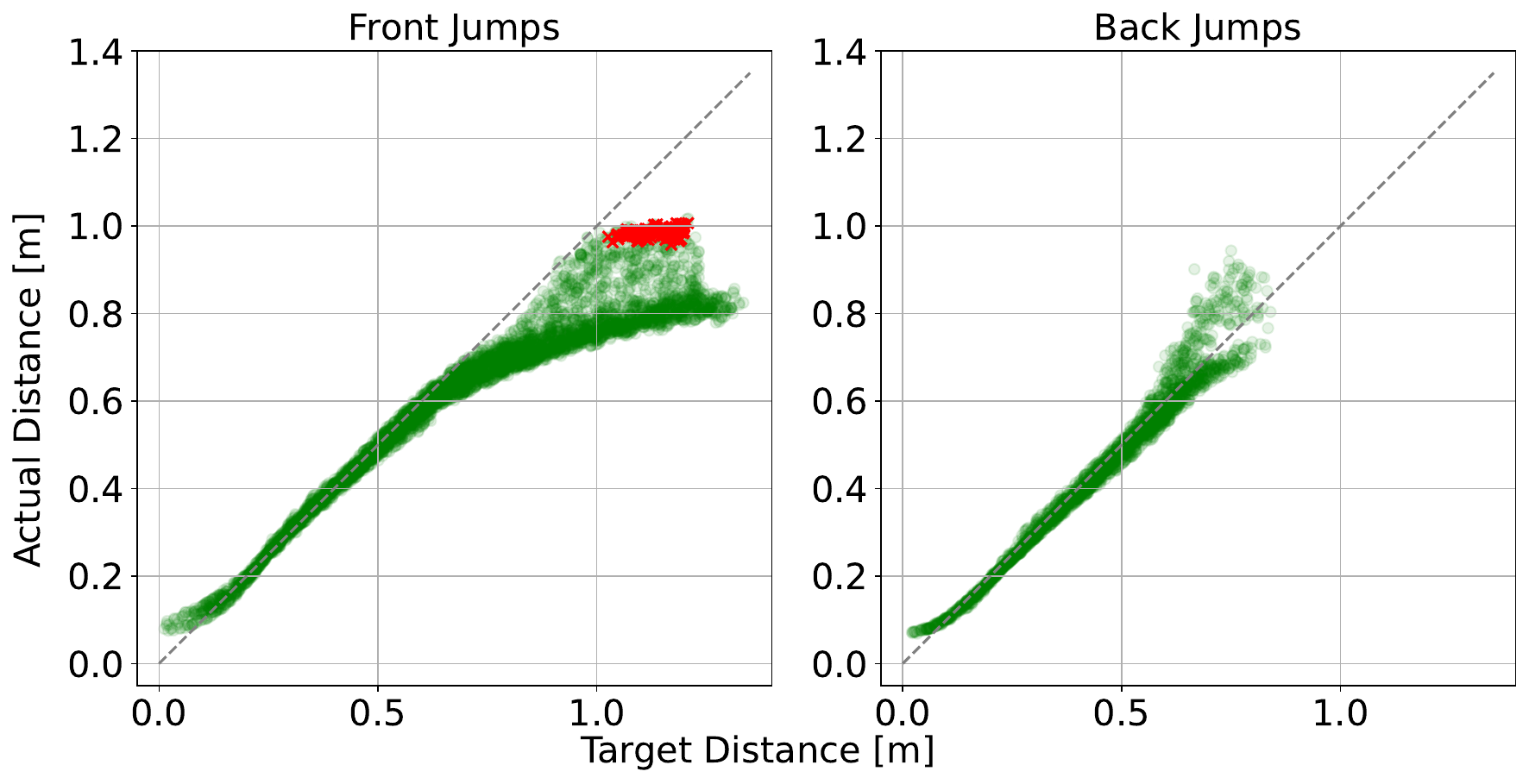}
    \caption{In green, target vs actual jump distance for both front (left plot)  and back (right plot) jumps using the \textit{omni-directional} jump policy with the Go1 robot. Jump failures are highlighted in red.}
    \label{fig:actual-vs-target}
\end{figure}
In an ideal scenario, the robot’s actual jump distances would align perfectly with the desired target, 
depicted by the dashed diagonal line in the plot.
The green shaded area corresponds to jumps executed without any collisions involving non-foot parts 
of the robot, representing the actual performance of the policy. The red samples are jumps that are failing.
For target distances up to 0.6 meters, the policy performs well, showing precise and consistent jumps indicating that the robot effectively tracks the desired trajectory.
The relatively narrow distribution of points in this range indicates that the robot maintains good accuracy as the target distance varies within this limit.
However, as the target distance exceeds 0.6 meters, the policy exhibits more cautious behavior. 
This conservatism is reflected in the robot's inability to execute larger jumps beyond this distance. 

This suggests that the policy prioritizes feasibility and safety over maximizing jump distance, 
aiming to ensure stable landings and minimize the risk of failure and is consistent with the trend 
observed in the reward plot, where minimal violations of constraints were recorded. Rather than 
pushing the system toward potentially unsafe physical limits, it favours reliability and consistency. 
Introducing a \gls{nn} focused on improving landing proficiency is expected to significantly enhance performance, particularly by enabling longer jumps. This is because rebounce effects will be handled by the dedicated landing policy, rather than requiring the jump policy to adopt a conservative approach.
Finally, there are 124 failure points around the 1-meter jump distance, equivalent to only the $1\%$ of the jumps. 
%
%
\subsubsection{Upward and Downward Jumps}
In this subsection, we evaluate the performance of the learned policy in executing jumps 
involving both positive and negative vertical displacements, therefore focusing on upward and downward jumps. 
These types of jumps are critical for robotic navigation on uneven terrains or in environments 
with varying surface elevations, where the ability to adapt to different heights 
is essential for effective mobility, and are mostly overlooked in the literature.
The tests were conducted within the samples in the test region as before, but with target locations 
where Z component was set  to different values along the z-axis.
These ranged from upward jumps to higher platforms (positive vertical displacement) to downward jumps onto lower surfaces (negative vertical displacement).
For each target location (x,y),  the maximum/minimum achievable height is recorded. The results of these tests were used to generate height maps (see Fig. \ref{fig:height-map}) for upward (left) and downward  jumps (right), respectively.

As shown in the left contour plot, the robot is capable of executing jumps to surfaces with up to 0.26 meters of vertical displacement. The total range of successful upward jumps lies between 0.16 meters and 0.28 meters, with some variation depending on the lateral displacement. Notably, the robot tends to struggle with upward jumps targeting locations in the back region, particularly when combined with significant lateral displacement. This behavior is expected due to the kinematic constraints of the robot in these specific configurations, where the required force distribution for pushing off is less efficient (given the leg shape), and the available range of motion becomes more limited.

In contrast, the right contour plot demonstrates that the robot can successfully perform downward 
jumps across its entire feasible region, with vertical displacements  up to -0.4 meters, which 
aligns with the bounds established during the training and test phase.  
This indicates that the robot could possibly  handle even larger drops. 
The greater effectiveness in downward jumps can be attributed to the assistance of 
gravity in achieving the target, making the kinematic demands less stringent compared to upward jumps.
\begin{figure}[htbp]
\centering
    \includegraphics[width=\columnwidth]{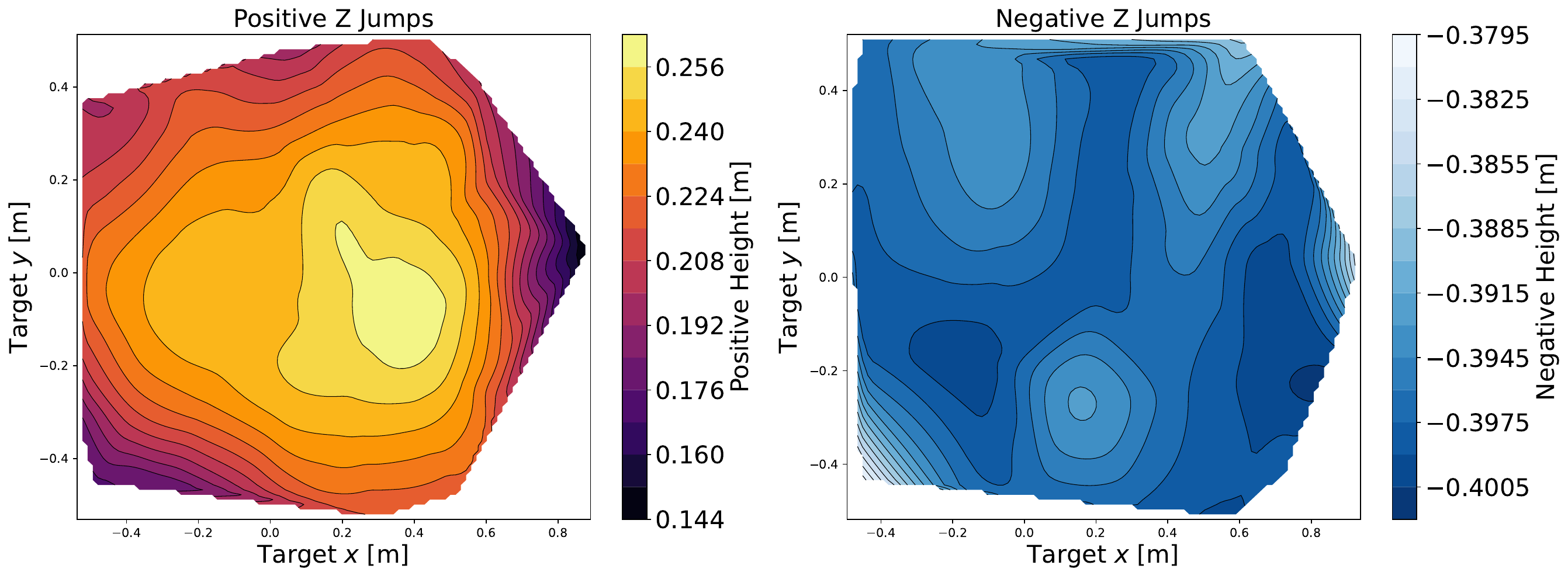}
    \caption{Height map for both upward jumps (on the left) and downward jumps (on the right) for the Go1 robot.}
    \label{fig:height-map}
\end{figure}
\subsubsection{In-place Rotation Jumps}
In this test, we performed a series of in-place jumps where each jump involved 
a different target  orientation change for the yaw component.   This capability is particularly beneficial for executing maneuvers that require quick rotations, such as twist jumps or adjustments to prepare for subsequent locomotion.
The test examines whether the learned policy can effectively control the robot’s orientation 
ensuring it reaches the desired landing position with a specified orientation.
The polar plot in Fig. \ref{fig:polar-twist} shows the variation of the landing orientation error (radius) 
as a function of the desired yaw angle. In-place jumps with no orientation change 
(origin of the polar plot) serve as a baseline, resulting in minimal orientation error, as expected. However, as the desired yaw angle increases in magnitude (both positively and negatively) the orientation error raises symmetrically. This symmetric behavior reflects the robot's  ability to handle yaw rotations in either direction.

The maximum observed yaw error is approximately 6 degrees, which is a satisfactory result, considering that it stays under $10\%$ of the desired yaw angle across a wide range of rotations. This level of error indicates that the learned policy is proficient at managing dynamic jumps involving significant reorientation, maintaining a reasonable level of accuracy even under challenging angular displacement scenarios. 
\begin{figure}[th!]
\centering
    \includegraphics[width=0.8\columnwidth]{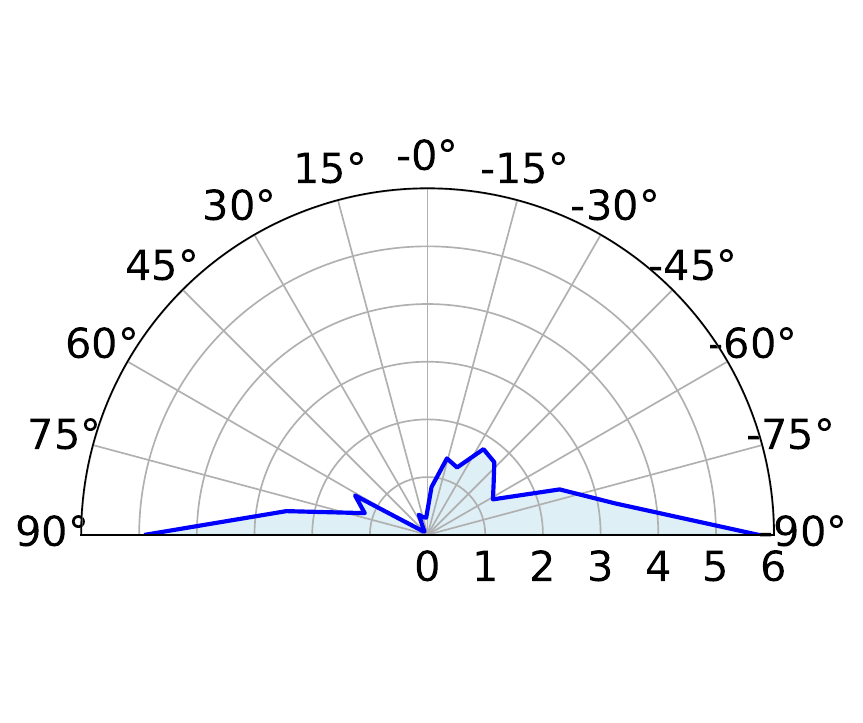}
    \caption{Polar plot of the landing orientation error (radius) vs the desired yaw angle, both expressed in degrees for the Go1 robot.}
    \vspace{-0.5cm}
    \label{fig:polar-twist}
\end{figure}
\subsection{Sim-to-sim evaluation and robustness to parameter variation}
In this section, we will test the trained policy  using a  simulation  framework (Locosim \cite{focchi2023locosim} based on Gazebo) 
different than the one used for training (Orbit).
Our intent is twofold: 1)  assess the impact of the sim-to-sim gap by evaluating how different simulation environments affect robot behavior and performance; 2)    systematically study how changes in physical parameters—such as mass distribution and joint friction affect the robot’s dynamic behavior, with a specific focus on the accuracy and reliability of jumping maneuvers.
To this extent we consider two types of jumps on flat terrain: a forward jump of 0.4 m (\textit{FWD}) and a composite diagonal jump (\textit{DIAG})  with target $c_{tg,x}$ = 0.3 m, $c_{tg,y}$ = 0.2 m, $\psi_{tg}$ = 45 deg; 3 types of tests: \textit{NOM} (no parametric variation), \textit{DV} (variations on joint damping) and  \textit{MV} (variations on mass); and two robot platforms (Go1, Aliengo). 
We performed 100 simulations for each tuple (robot, jump type, test). For the cases \textbf{DV} and \textbf{MV}  with parametric variations  sampled from a continuous uniform distribution $\mathcal{U}(0,1)$ as follows:
\begin{equation}
\begin{split}
&m_{\text{var}} = m_0 + \underbrace{m_0  p}_{\Delta m_{\text{max}}}  \mathcal{U}(-1,1)\\
&d_{\text{var}} = d_{\text{max}} \cdot \mathcal{U}(0,1)
\end{split}
\end{equation}
Here, $m_0$ is the default mass, $p=0.5$ is the perturbation percentage, 
and $m_{\text{var}}$ and  $d_{\text{var}}$ are the modified mass and damping, respectively, (the default  damping  being 0). 
Then, for each tuple, we computed the mean and the standard deviation of the  
errors on the  landing position  $e_{x}$, $e_{y}$ and orientation $e_{\psi}$   as reported in Table \ref{tab:statistics}.
\ifdefined\springer
	\begin{table}[th!]
	\centering
	\caption{Statistics of accuracy VS  parameter uncertainty}
	\begin{tabular}{c|c|c | c | c | c | c } 
	    \hline
	     \textbf{\#}& \textbf{Robot} & \textbf{Jump Type} & \textbf{Test} & $\mathbf{ e_x }$ &  $\mathbf{ e_y }$ & $\mathbf{e_{\psi}}$\\
	    \hline\hline
	   T1& Go1     & \textit{FWD}  & \textbf{\textit{NOM}} &   0.013  $\pm$  0.002 & X                   &  X\\
	   T2& Go1     & \textit{FWD}   & \textit{DV}           &   0.059  $\pm$  0.023 & X                   &  X\\
	   T3& Go1     & \textit{FWD}   & \textit{MV}           &   0.014  $\pm$  0.007 & X                   &  X\\
	   T4& Go1     & \textit{DIAG}  & \textbf{\textit{NOM}} &   0.019  $\pm$  0.000 & -0.024 $\pm$ 0.001   &  1.523 $\pm$ 0.128\\
	   T5& Go1     & \textit{DIAG}   & \textit{DV}           &   0.032  $\pm$  0.011 & -0.003 $\pm$ 0.012   &  3.982 $\pm$ 2.237\\
	   T6& Go1     & \textit{DIAG}   & \textit{MV}           &   0.019  $\pm$  0.001 & -0.023 $\pm$ 0.005   &  1.632 $\pm$ 1.137\\
	   T7& Aliengo & \textit{FWD}   & \textit{NOM}          &   0.010  $\pm$  0.017 &  X                   &  X  \\
	   T8& Aliengo & \textit{FWD}   & \textit{DV}           &   0.044  $\pm$  0.024 &  X                   &  X  \\
	   T9& Aliengo & \textit{FWD}   & \textit{MV}           &   0.014  $\pm$  0.024 &  X                   &  X  \\
	   T10& Aliengo & \textit{DIAG}   & \textit{NOM}          &   0.005  $\pm$  0.013 & -0.005 $\pm$ 0.028   & 10.992 $\pm$ 2.871\\
	   T11& Aliengo & \textit{DIAG}   & \textit{DV}           &   0.028  $\pm$  0.015 &  0.020 $\pm$ 0.024   &  9.166 $\pm$ 3.577\\
	   T12& Aliengo & \textit{DIAG}   & \textit{MV}           &   0.009  $\pm$  0.016 &  0.003 $\pm$ 0.030   & 11.304 $\pm$ 3.520\\
	    \hline\hline
	\end{tabular}
	\label{tab:statistics}
	\end{table}
\else
	\begin{table}[th!]
		\centering
		\caption{Statistics of accuracy VS  parameter uncertainty}
		\resizebox{\columnwidth}{!} {
			\begin{tabular}{c|c|c | c | c | c | c } 
				\hline
				\textbf{\#}& \textbf{Robot} & \textbf{Jump Type} & \textbf{Test} & $\mathbf{ e_x }$ &  $\mathbf{ e_y }$ & $\mathbf{e_{\psi}}$\\
				\hline\hline
				T1& Go1     & \textit{FWD}  & \textbf{\textit{NOM}} &   0.013  $\pm$  0.002 & X                   &  X\\
				T2& Go1     & \textit{FWD}   & \textit{DV}           &   0.059  $\pm$  0.023 & X                   &  X\\
				T3& Go1     & \textit{FWD}   & \textit{MV}           &   0.014  $\pm$  0.007 & X                   &  X\\
				T4& Go1     & \textit{DIAG}  & \textbf{\textit{NOM}} &   0.019  $\pm$  0.000 & -0.024 $\pm$ 0.001   &  1.523 $\pm$ 0.128\\
				T5& Go1     & \textit{DIAG}   & \textit{DV}           &   0.032  $\pm$  0.011 & -0.003 $\pm$ 0.012   &  3.982 $\pm$ 2.237\\
				T6& Go1     & \textit{DIAG}   & \textit{MV}           &   0.019  $\pm$  0.001 & -0.023 $\pm$ 0.005   &  1.632 $\pm$ 1.137\\
				T7& Aliengo & \textit{FWD}   & \textit{NOM}          &   0.010  $\pm$  0.017 &  X                   &  X  \\
				T8& Aliengo & \textit{FWD}   & \textit{DV}           &   0.044  $\pm$  0.024 &  X                   &  X  \\
				T9& Aliengo & \textit{FWD}   & \textit{MV}           &   0.014  $\pm$  0.024 &  X                   &  X  \\
				T10& Aliengo & \textit{DIAG}   & \textit{NOM}          &   0.005  $\pm$  0.013 & -0.005 $\pm$ 0.028   & 10.992 $\pm$ 2.871\\
				T11& Aliengo & \textit{DIAG}   & \textit{DV}           &   0.028  $\pm$  0.015 &  0.020 $\pm$ 0.024   &  9.166 $\pm$ 3.577\\
				T12& Aliengo & \textit{DIAG}   & \textit{MV}           &   0.009  $\pm$  0.016 &  0.003 $\pm$ 0.030   & 11.304 $\pm$ 3.520\\
				\hline\hline
		\end{tabular}}
		\label{tab:statistics}
	\end{center}
	\end{table}
\fi

Inspecting the values for the nominal test (T1 and T4) we can observe a good matching for a \textit{FWD} jump (see Fig. \ref{fig:feasible-region}) in terms of the position error ($\approx 1.3$ cm) and for a \textit{DIAG} jump  (see Fig. \ref{fig:polar-twist}) in terms of the orientation error (1.5 deg) demonstrating a low sim-to-sim gap.
Regarding the tests involving changes in dynamic parameters of the robot there is a very little influence of the mass (despite its variation is up to 50\%)  while a higher influence is observed for the damping with the position error in the \textit{FWD} jump which is increasing from 1.5 cm (T1) to  5.9 cm (T2),  which is, in percentage, still below the 15\% of the whole jump length. 
These results are particularly relevant for the \textit{MV} test, where gravity compensation was computed using the default mass value, resulting in a significant deviation from real conditions, showing that leveraging a low-level controllee, our method exhibits good robustness to parameter uncertainty without the need to add \textit{domain randomization} during the training. 
Similar results are observed on the orientation error that remains between 1.5 and 4 degrees in the case of \textit{DIAG} jump (T4 and T5). 
Equivalent outcomes were obtained with the Aliengo robot (for which the policy was specifically retrained) with a maximum landing error for a \textit{FWD} jump of 4.4 cm in the case of \textit{DV} (T8) with respect to 1 cm of the \textit{NOM} (T7). 
\subsection{Joint Feasibility of the Learned Trajectory}
In this Section we want to assess  both the quality of the  tracking of the
generated trajectory and their physical feasibility considering 
constraints on joint position, velocity and torque.   
Fig. \ref{fig:sim_tracking} presents the desired and actual trajectories for \gls{com} position, velocity, yaw, and yaw rate during a composite \textit{DIAG} jump with the Go1 robot in a simulation with Locosim, being the \textit{DIAG} jump  a good template to evaluate tracking performance because it involves motion along multiple directions.  
The tracking is good during the whole thrusting phase which is highlighted by the shaded area. Accurate tracking of joint positions and velocities further demonstrates the controller’s ability to maintain stability throughout the jump.
Additionally, in the accompanying video (figure not included here due to space constraints) we assess whether any kinematic, velocity, or torque limit is exceeded. As shown, the system consistently operates within all joint limits during the thrust phase indicating that, even under demanding conditions, the policy respects the robot’s physical constraints. 
This physical compliance results in a conservative behavior, 
particularly evident in longer-distance jumps, where the policy prioritizes safety (c.f Section \ref{sec:socomparison}).
Overall, the results confirm the method’s physical feasibility, whose success is 
largely due to the integration of domain knowledge into the reward design. 
%
\begin{figure}[th!]
\centering
    \includegraphics[width=1.0\columnwidth]{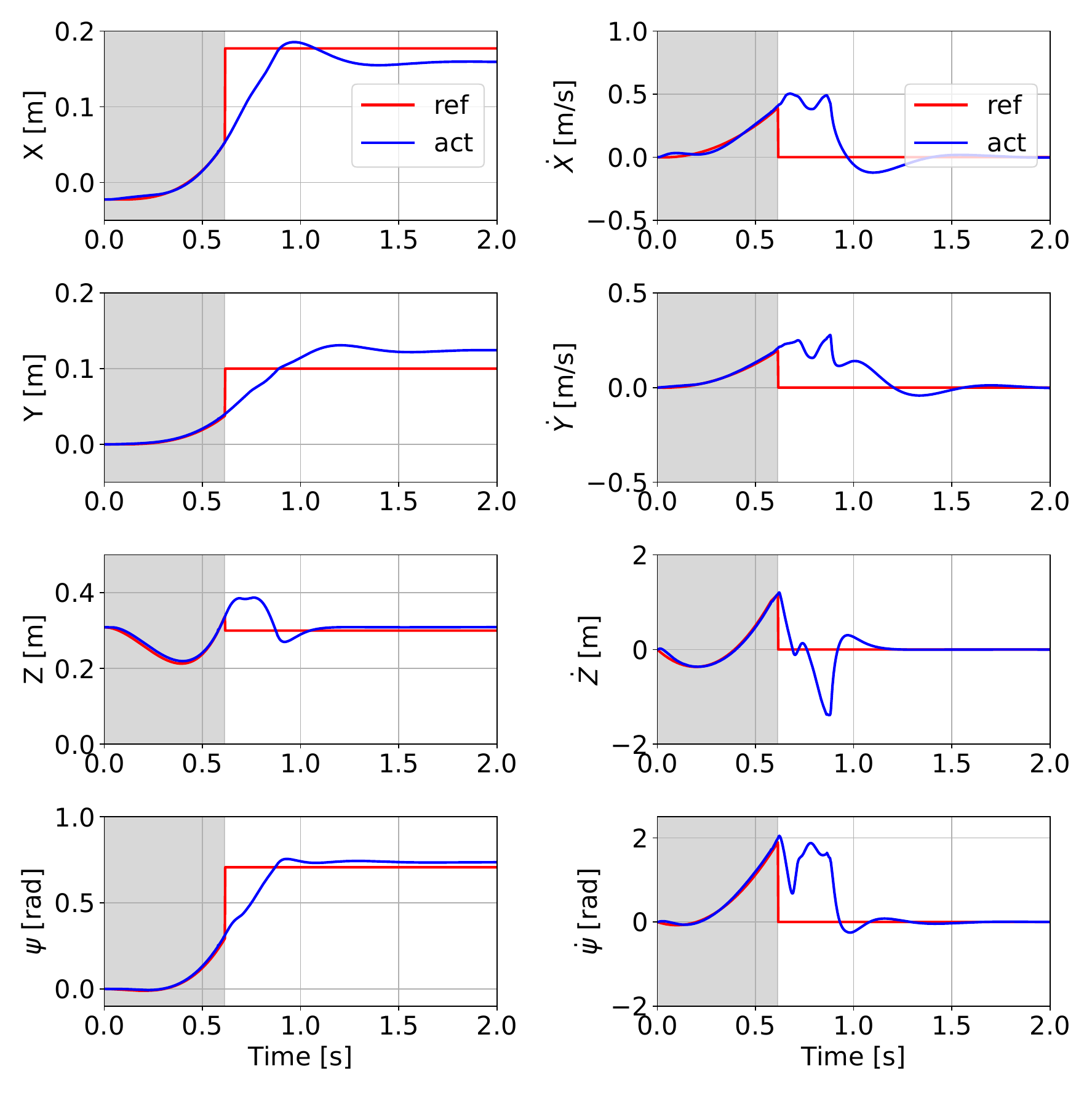}
    \caption{ \textit{Simulation}. Tracking performances for the \textit{DIAG} jump with the Go1 robot: \gls{com} position tracking (upper-left plots) and orientation (yaw) tracking (lower-left plot). (right) tracking of  \gls{com} velocity (upper-right plots) and of the yaw rate (lower-right plot). References are in red and are switched to target values after lift-off, the actual measurements are in blue. The shaded area highlights the thrust phase.}
    \label{fig:sim_tracking}
\end{figure}

\subsection{Comparison with SOA Approach}
\label{sec:socomparison}
In this section we compare our method with the state-of-the-art \gls{e2e} approach developed by Atanassov et al. \cite{atanassov2024curriculumbased}, which  is currently regarded as the benchmark for RL-based jumping, offering a valuable reference for evaluating our performance. Briefly, their method involves dividing the learning process into three curriculum stages, with the same policy, represented by a neural network, being refined sequentially. The first stage focuses on in-place jumps with variable positive height displacements. The second stage introduces forward and lateral displacement jumps, while the final stage involves obstacle-aware jumps, where the robot learns to jump onto a given obstacle.
\begin{figure}[th!]
\centering
    \includegraphics[width=0.52\columnwidth]{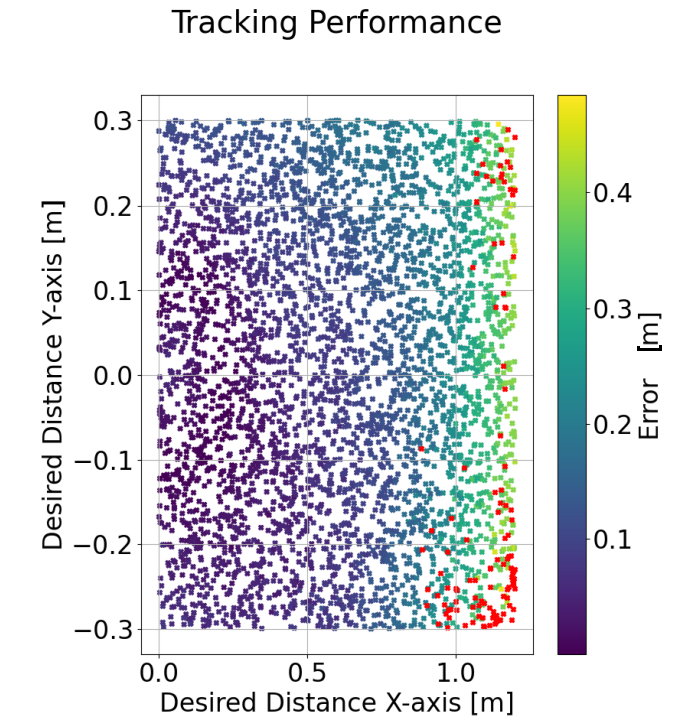}
    \includegraphics[width=0.45\columnwidth]{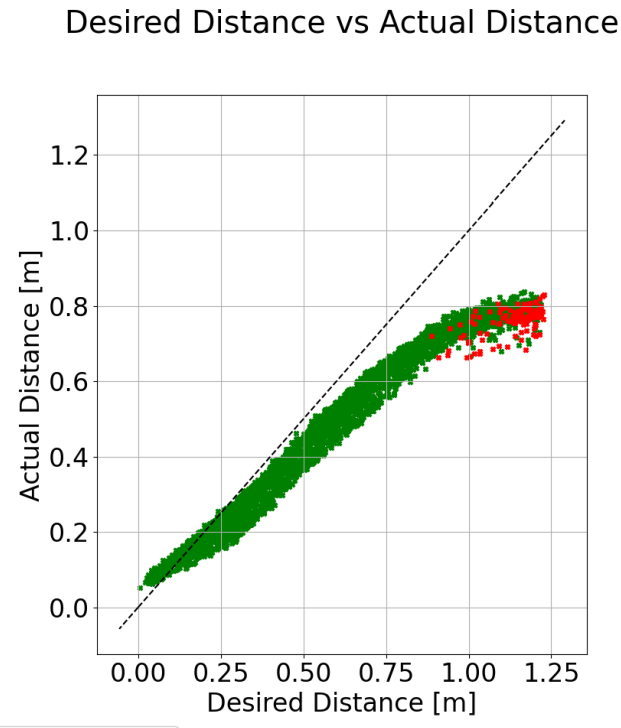}
    \caption{Tracking performances reported in Atassanov et al. \cite{atanassov2024curriculumbased}, with a policy trained for solely forward jumps.   Unsuccessful jumps are depicted in red.}
    \label{fig:soa}
\end{figure}
\begin{figure}[htbp]
\centering
    \includegraphics[width=1.0\columnwidth]{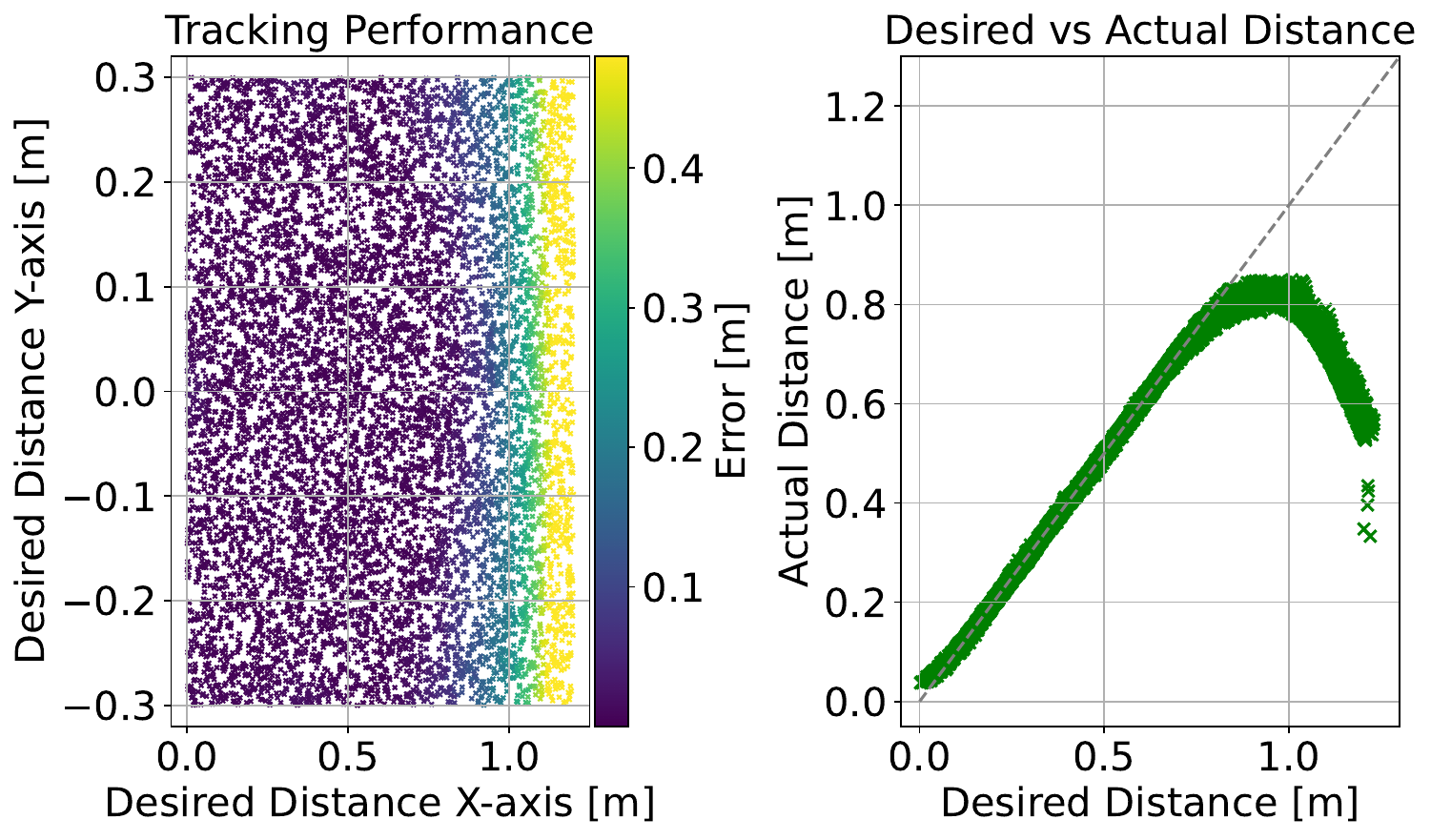}
    \caption{Tracking performances for the Go1 robot reported for our approach with a policy trained for solely forward jumps. }
    \label{fig:our-soa}
\end{figure}
The work was developed using an early version of Isaac-Gym in combination with a predecessor framework of Orbit, and the total training time took approximately 10 hours. Like our work, their training utilized parallelization with 4096 simulated robots running simultaneously. However, in their case, each policy update was computed after 24 environment steps, resulting in a higher data requirement compared to our method. The entire training process required 23000 iterations, with 3000 iterations for the first stage and 10000 for each of the remaining two stages. This highlights the high data requirements of \gls{e2e} approaches compared to our more sample-efficient \gls{grl} method that converges after only 2000 training steps (Fig. \ref{fig:rew-pen}).
In Fig. \ref{fig:soa}, we present the results of their approach for a policy trained specifically for forward jumps with the Go1 robot, while in Fig. \ref{fig:our-soa}, we show our results with our policy using a similar form factor for equal comparison. For consistency, we retrained our policy for forward jumps by utilizing their training flat region composed of an X-range of (0-1.2) meters and a Y-range of (-0.3, 0.3) meters. 
A key observation is that, in their approach, distant jumps combined with lateral 
displacements often result in failed jumps due to improper contacts that are not made by the feet. 
In contrast, our approach avoids such failures by ensuring the respect of system constraints also in the task space,
featuring the same maximum jump distance compared to their method but with a more \textit{uniform}
performance within the previously defined feasible region, as evidenced by the proximity of our jump distributions to the ideal jumping behavior (diagonal line).
Although their plot does not explicitly report jump heights, in the paper they declare
that their method was tested with jump heights of only 0.05 and 0.1 meters, whereas our method demonstrated the ability to reach up to 0.26 meters for forward jumps. This highlights our method’s superior capability for handling higher vertical displacements. 
Note that with respect to the omnidirectional policy, where the maximum accurate landing was limited to 0.6 m, with the forward policy we are able to attain longer jumps up to 0.8 m.

To summarize, our method outperforms Atanassov’s in several aspects: 
(1) it delivers higher accuracy, as seen in the left plot, where the region with less than 10 cm error extends uniformly along the Y-axis up to 80 cm jumps; (2) this accuracy is also reflected in the right plot, where samples are more tightly distributed around the ideal line; (3) it enables higher elevations; (4) the jump trajectory can be verified before the execution. Finally, unlike the omni-directional jump policy, our method produces no failures (i.e., no red samples).
Our method adopts a more conservative strategy that prioritizes safety and accuracy, with the added advantage of requiring significantly fewer samples and avoiding complex multi-stage reward engineering. This makes our approach not only more data-efficient but also more straightforward to implement and train, without sacrificing real-world feasibility and performance. This is supported by the evidence that their approach incorporates 25 distinct reward terms tailored to each curriculum stage. In contrast, our approach uses only 15 reward terms and a single-stage learning.
\subsection{Zero-Shot on Different Robotic Platform}
To demonstrate the generalization of our approach, we selected a quadruped platform similar to the Go1 and performed a zero-shot jump with this platform, making only the necessary low-level control adjustments. The selected platform is the Unitree Aliengo, which weighs approximately 24 kg and features longer leg links with higher torque limits compared to the Go1. As shown in the accompanying video,  we tested on Aliengo a forward jump of 0.6 m employing the policy trained on Go1, achieving a landing error of 0.1 m while satisfying all system constraints related to the robot's characteristics. 
This underscores the potential for directly applying the trained policy to robotic systems with similar specifications, or alternatively, for rapidly fine-tuning the pre-trained policy (without retraining from scratch) to adapt to differences in actuation, mass, and dimensional characteristics across platforms. By defining both the action space and observation space in task-space, our approach contributes to making the policy largely independent of the specific robotic platform. 
%
%
\subsection{Experiments with the real quadruped}
\label{sec:real_exps}
In this section, we report real experiments of the \textit{FWD}, \textit{BWD} and \textit{DIAG} jumps performed both with the Aliengo and Go1 robots. \textit{BWD} jump is a 40 cm jump backwards. 
The experiments were repeated 10 times for each jump type 
and are shown in the accompanying video. 
The estimation of the \gls{com} position is implemented by leg odometry assuming that
the feet remain in stance during the whole thrust while the orientation comes from \gls{imu} readings. 
We measured the landing location with a measurement tape because the odometry would provide non-meaningful estimates after lift-off, and a Motion Capture System was not available. 
Because of the short duration of the jump and the high acceleration involved, it was impossible to estimate the \textit{apex} and \textit{touchdown} conditions  (as customary done in robotics)  via contact force estimation or contact switches. In the first case, it was difficult to set a fixed threshold on the ground reaction forces to assess the contact condition because the high joint acceleration due to the leg retraction created fictitious force estimation that triggered early touchdown when the robot was still in the air. In the second case, the contact sensor in the Unitree robot was unacceptably slow ($\approx$100ms) in detecting the  lift-off, 
because it is based on air pressure sensors, which are inherently subject to pressure dynamics. 

The key to achieve successful experiments was leveraging our physically 
informed action parametrization to \textit{heuristically} estimate the apex and touchdown instants 
by using ballistic equations (see Section \ref{sec:physical_check}).
We report in Table \ref{tab:exps_statistics} the statistics (mean and std. deviation) of the position error (in the x direction)   and orientation errors (in yaw direction) for the three types of jumps (performing 10 experiments for each),  with the Aliengo robot. 
We report also the \textit{FWD} jump with the Go1 robot, for comparison. 
\ifdefined\springer
	\begin{table}[th!]
	\centering
	\caption{Statistics of accuracy in real experiments}
	\begin{tabular}{c| c|c | c | c | c  } 
	    \hline
	    \textbf{ID} & \textbf{Robot} & \textbf{Jump Type}   & $\mathbf{e_x} [m]$ & $\mathbf{e_y}[m]$ & $\mathbf{e_{\psi}}[rad]$\\
	    \hline\hline    
	    T1 &Aliengo & FWD      &  0.046 $\pm$ 0.01  &  N/A & N/A \\
	    T2 &Aliengo & BWD      &  0.071  $\pm$ 0.015&  N/A       & N/A \\
	    T3 &Aliengo & DIAG     &  -0.057  $\pm$ 0.017 &  -0.16  $\pm$ 0.026  & -12.17 $\pm$ 3.1 \\
	    T4 &Go1     & FWD      &  0.01  $\pm$ 0.017 &  N/A & N/A \\
	    \hline\hline
	\end{tabular}
	\label{tab:exps_statistics}
	\end{table}
\else
	\begin{table}[th!]
		\centering
		\caption{Statistics of accuracy in real experiments}
		\resizebox{\columnwidth}{!} {
			\begin{tabular}{c| c|c | c | c | c  } 
				\hline
				\textbf{ID} & \textbf{Robot} & \textbf{Jump Type}   & $\mathbf{e_x} [m]$ & $\mathbf{e_y}[m]$ & $\mathbf{e_{\psi}}[rad]$\\
				\hline\hline    
				T1 &Aliengo & FWD      &  0.046 $\pm$ 0.01  &  N/A & N/A \\
				T2 &Aliengo & BWD      &  0.071  $\pm$ 0.015&  N/A       & N/A \\
				T3 &Aliengo & DIAG     &  -0.057  $\pm$ 0.017 &  -0.16  $\pm$ 0.026  & -12.17 $\pm$ 3.1 \\
				T4 &Go1     & FWD      &  0.01  $\pm$ 0.017 &  N/A & N/A \\
				\hline\hline
		\end{tabular}}
		\label{tab:exps_statistics}
	\end{table}	
\fi
The mean error  for the \textit{FWD} jump with the Go1 robot is quite close to the simulation value  (T4) (c.f. Table \ref{tab:statistics}) 
but more than three times higher for the Aliengo robot (T1) with the robot undershooting the target. 
Conversely, by analyzing the error direction in \textit{BWD} jumps, the Aliengo robot tends to overshoot the target.
Regarding the composite \textit{DIAG} jump, the robot overshoots the target too in the 3 directions, resulting, for orientation, 
in a mean error of around -12 deg, which is  26\% of the target 45-degree orientation.  
Despite the fact that the error offsets are not negligible, the reader should note that their standard deviation is low, which showcases the high reliability of the approach in terms of high
repeatability. We believe the source of these errors is due to friction which has not been properly compensated.  The bad tracking on the $y$ direction is responsible for the 16 cm average error that we obtain in that direction. Training an actuator \gls{nn} \cite{hwangbo2019learning} can be effective in removing these non-idealitites from the low-level controller.   
We also showcase the feasibility of the generated reference trajectories checking the tracking performances in Fig. \ref{fig:real_tracking} for the \textit{DIAG} jump. Solely the thrust phase is reported because the odometry would provide non-meaningful results after lift-off. We can observe that the tracking is reasonable at the position level (left plot) and good at the velocity level (right plots), both for \gls{com} and base orientation. Being the lift-off velocity crucial for most of the landing accuracy.
\begin{figure}[th!]
\centering
    \includegraphics[width=1.0\columnwidth]{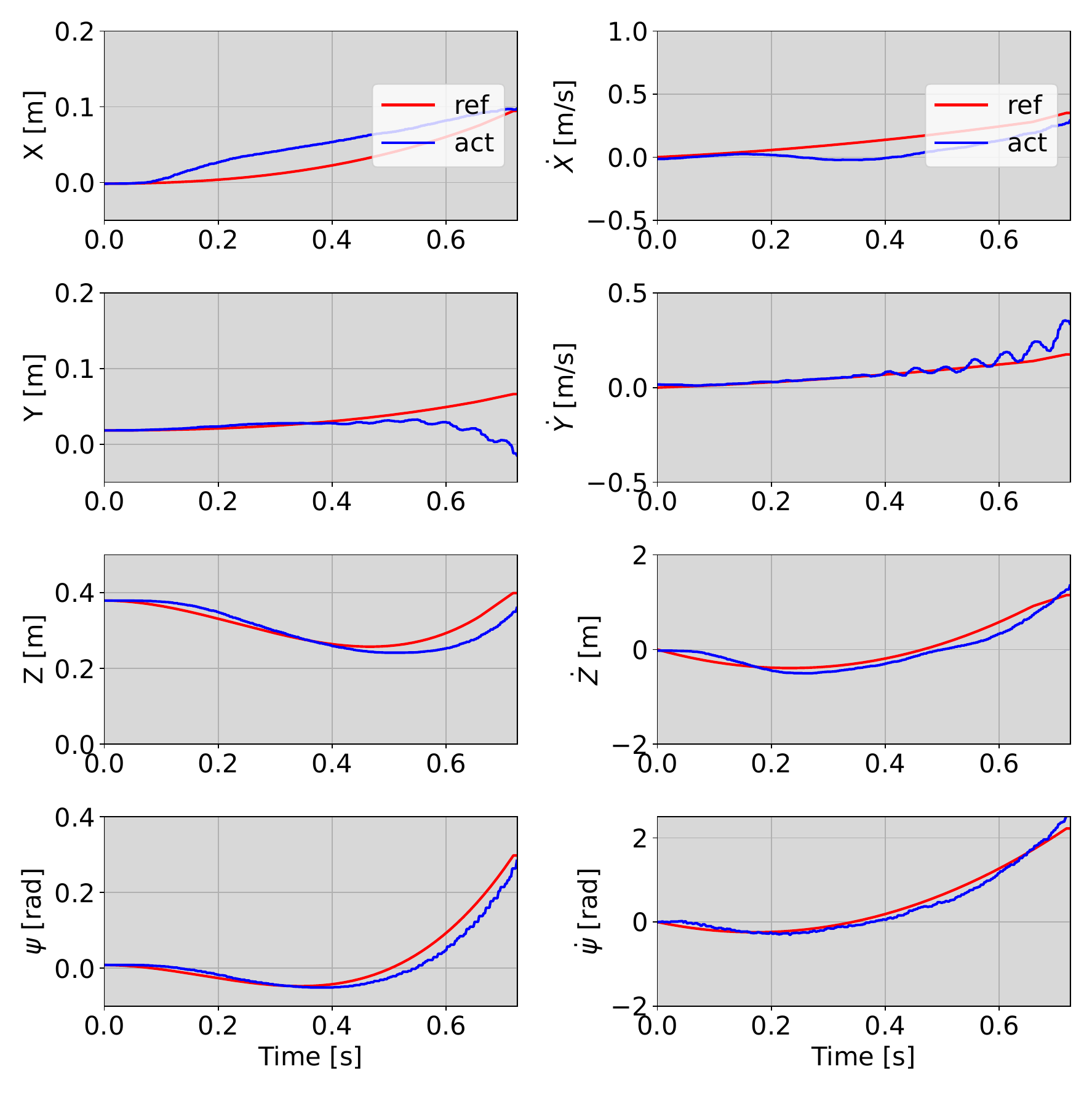}
    \caption{\textit{Experiments}. Tracking performances during the \textit{thrust} phase, for the \textit{DIAG} jump with the Aliengo robot: \gls{com} position tracking (first 3 plots-left) and orientation (yaw) tracking (lower-left plot). The \gls{com} velocity is shown in the upper-right plots whilst the yaw rate in the lower-right plot. References are in red while actual measurements are in blue. }
    \label{fig:real_tracking}
\end{figure}
%
%
%
%
\section{Conclusions}
\label{sec:conclusion}


%


In this work, we propose a novel approach for generating jump
manoeuvres in legged robots. The methodology centres on the notion of
Guided Reinforcement Learning, which we interpret as the injection of
physical intuition into the learning process.

In particular, we use a geometric parametrisation of the curve
followed by the centre of mass during the thrust phase, and
concentrate the learning process into a small set of
parameters—achieving greater sample-efficiency and reduced computation
time compared to customary \gls{e2e} approaches. We also employ classical control to track the generated trajectory at
the joint level, resulting in a good degree of robustness to parameter
uncertainty. Orientation is managed via a separate trajectory, a
solution that enables complex and dynamic manoeuvres such as in-place
twist jumps for rapid changes in heading. The learned policy shows a remarkable ability to handle jumps of
varying heights—including downward jumps—and to perform pure yaw
rotations. The generated trajectories successfully manage orientation
changes while consistently respecting the system’s physical
constraints.


Contrary to previous approaches~\cite{feng} in which morphology and
control strategy are tightly bound, we have achieved a good degree of
platform independence. A crucial choice to obtain this result was
excluding joint-level information from observations and formulating
actions directly in task space, which allowed us to the manage
the differences between the platforms by simply using inverse
kinematics.  As a direct validation of this idea, we were able to
demonstrate zero-shot transfer of the policy learned for the Go1
robot to the Aliengo platform.




%
%

Compared to state-of-the-art end-to-end reinforcement learning
approaches~\cite{atanassov2024curriculumbased}, our method achieves
comparable performance on jump tasks while requiring significantly
less training—only 2,000 episodes versus millions—marking a
substantial improvement in training efficiency.  This gain in
efficiency does not come at the expense of performance: both the
accuracy of the jumps and the achieved elevation outperform the
baseline~\cite{atanassov2024curriculumbased}.

In spite of the evident success of the application of \gls{grl} to the challenging task of managing robot jumps, several limitations remain that will require future research efforts. Some of the most important are the following:

\begin{enumerate}
    \item The current method is purely kinematic. We believe that controlling the dynamics in torque-controlled systems would produce significant performance improvements.

    \item The policy has been trained assuming zero initial velocity, which restricts its application in dynamic scenarios such as multiple jumping tasks. Generalising to arbitrary initial velocities will be an obvious research direction for future work.

    \item We are seeking to extend the state representation to include the position of the feet relative to the base frame. This improvement would enable operations in adverse environments, such as sloped surfaces.

    \item Including a $\Delta \tau$-network, trained via an \gls{e2e} approach, could allow us to adjust the joint configuration at each time step and support airborne reorientation, helping to mitigate the impact during the landing phase.

    \item We plan to increase the sample efficiency and performance of our method even further by considering data augmentation techniques that exploit the symmetries of the quadruped robot and the task \cite{ordonez2025morphosymm, suhuang2024leveraging}.

    \item Considering training scenarios where the robot rears on two legs could lead to more dynamic pre-jump strategies and expand the action repertoire. Obvious benefits in terms of robustness and generality could be achieved by including deformable terrain—such as trampolines or soft ground—within the training environment.

    \item Finally, we aim to investigate the integration of reinforcement learning with numerical optimisation techniques, to combine the flexibility of data-driven methods with the precision and interpretability of model-based approaches.
\end{enumerate}

\section*{Funding Declaration}
This research received no specific grant from any funding agency in the public, commercial, or not-for-profit sectors.
\begin{appendices}
\section{Derivative of a B\'ezier function}
\label{sec:bez_der}
\begin{lemma}
The derivative of a Bézier curve $\mathbf{B}(t) =\sum_{i=0}^{n}b_{i}^n(t)\mathbf{P}_{i}$ of degree $n$ is a  Bézier curve of order $n - 1$ of the form \cite{bernstein_derivative_proof}: 
\begin{equation}
    \label{eq:bez-der2}
    \mathbf{\dot{B}}(t) = \sum_{i=0}^{n-1}b_{i}^{n-1}(t)\mathbf{P^{\prime}}_{i} \;\;\;\; 0 \leq t \leq 1
\end{equation}
with control points defined as
\begin{equation}
    \mathbf{P^{\prime}}_{i} = n\left(\mathbf{P}_{i+1}-\mathbf{P}_{i}\right)
\end{equation}
\end{lemma}
\begin{proof}
To calculate the derivative of $\mathbf{B}(t)$, we begin by noting that since the control points $\mathbf{P}_{i}$ are constant and independent of time $t$, the calculation simplifies to differentiating the Bernstein basis polynomials, 
\begin{equation}
    \dot{\vect{B}}(t) =\sum_{i=0}^{n} \mathbf{P}_{i}\frac{\mathrm{d}b_{i}^n}{\mathrm{d}t}(t)
\end{equation}
where it can be proved that the derivative can be written as a combination of two polynomials of lower degree:
\begin{equation}
    \frac{\mathrm{d}b_{i}^n}{\mathrm{d}t}(t) = \dot{b}_{i}^n(t) = n\left(b_{i-1}^{n-1}(t)-b_{i}^{n-1}(t)\right)
\end{equation}
Hence, we can then compute the derivative of the curve obtaining:
\begin{equation}
     \mathbf{\dot{B}}(t) = \sum_{i=0}^{n} \mathbf{P}_{i} n \left[b_{i-1}^{n-1}(t) - b_{i}^{n-1}(t)\right] 
\end{equation}
With the convention $b_{-1}^{n-1}=b_{n}^{n-1}=0$ that comes from the fact the binomial coefficients in the Bernstein Polynomials vanish outside their natural range.  Splitting and regrouping the sums we get
\begin{equation}
     \mathbf{\dot{B}}(t) = n\left[ \sum_{i=0}^{n} \mathbf{P}_{i}   b_{i-1}^{n-1}(t) 
                                - \sum_{i=0}^{n} \mathbf{P}_{i} b_{i}^{n-1}(t)\right] 
\end{equation}
Now shifting the index in the first sum (i.e. let $j=i - 1$) and considering $b_{-1}^{n-1}=0$:
\small{\begin{equation}
 \sum_{i=0}^{n} \mathbf{P}_{i}  b_{i-1}^{n-1}(t) = 
 \sum_{j=-1}^{n-1} \mathbf{P}_{j+1} b_{j}^{n-1}(t) = 
 \sum_{j=0}^{n-1} \mathbf{P}_{j+1} b_{j}^{n-1}(t) 
\end{equation}}
a similar argument can be followed for the second term in the sum using $b_{n}^{n-1}=0$. Then, putting all together:
\begin{equation}
\begin{split}
        \mathbf{\dot{B}}(t) &= n\left[ \sum_{i=0}^{n-1} \mathbf{P}_{i+1}  b_{i}^{n-1}(t) 
                                - \sum_{i=0}^{n-1} \mathbf{P}_{i} b_{i}^{n-1}(t)\right]  \\
                            & = n \sum_{i=0}^{n-1} \left(\mathbf{P}_{i+1}-\mathbf{P}_{i}\right) b_{i}^{n-1} 
\end{split}
\label{eq:bez_derivative}
\end{equation}
\end{proof}
\section{Explicit form  of a cubic  B\'ezier curve}
\label{sec:explicit}
For clarity in notation, we will define the normalized time $\hat{t}$ as follows:
\begin{equation}
    \hat{t}= \frac{t}{T_{th}} \;\;\;\; 0 \leq t \leq T_{th} 
\end{equation}
The explicit forms for both the cubic Bézier curve and its 
corresponding quadratic derivative (computed from \eqref{eq:bez_derivative}) are as follows:
\small{
\begin{align}
\begin{split}
    \mathbf{B}(t) &= (1-\hat{t})^{3}\mathbf{P}_{0}+3(1-\hat{t})^{2}\hat{t}\mathbf{P}_{1}+3(1-\hat{t})\hat{t}^{2}\mathbf{P}_{2}+\hat{t}^{3}\mathbf{P}_{2}\\
    \mathbf{\dot{B}}(t)&= \frac{3}{T_{th}}(1-\hat{t})^{2}(\mathbf{P}_{1}-\mathbf{P}_{0})+\frac{6}{T_{th}}(1-\hat{t})\hat{t}(\mathbf{P}_{2}-\mathbf{P}_{1})
    \\ &~~~+\frac{3}{T_{th}}\hat{t}^{2}(\mathbf{P}_{3}-\mathbf{P}_{2}) \\
        & = (1-\hat{t})^{2}\mathbf{P^{\prime}}_{0}+2(1-\hat{t})\hat{t}\mathbf{P^{\prime}}_{1}+\hat{t}^{2}\mathbf{P^{\prime}}_{2}
\end{split}
\label{eq:expl-bez}
\end{align}
}
\end{appendices}
\small

\ifdefined\springer
\else
	\bibliographystyle{style/IEEEtran} 
\fi

\bibliography{references/references}
\end{document}